\documentclass[twoside,11pt]{article}

\usepackage{jmlr2e}

\usepackage{amsmath}
\usepackage{amsfonts}
\usepackage{amssymb}
\usepackage{color}
\usepackage{algorithm}
\usepackage{algorithmic}
\usepackage{url}
\usepackage{enumerate}
\usepackage{booktabs}

\newcommand{\NN}[1]{{\N}_{#1}}
\DeclareMathOperator*{\argmax}{\arg\!\max} 
\DeclareMathOperator*{\argmin}{\arg\!\min} 
 
\DeclareMathOperator{\prox}{prox} 
\DeclareMathOperator{\supp}{supp}
\let\inf\relax \DeclareMathOperator*\inf{\vphantom{p}inf}
\let\sup\relax \DeclareMathOperator*\sup{\vphantom{p}sup}

\newcommand{\trans}{^{\scriptscriptstyle \top}}

\newcommand{\tr}{{\rm tr}}

\newcommand{\R}{{\mathbb R}}
\newcommand{\G}{{\cal G}}

\newcommand{\WW}{{V}}
\newcommand{\ww}{{v}}

\newcommand{\card}{{\rm card}}

\newcommand{\N}{{\mathbb N}}

\newcommand{\ind}{{\ell}}

\newcommand{\beq}{\begin{equation}}
\newcommand{\eeq}{\end{equation}} 
\newcommand{\bea}{\begin{eqnarray}}
\newcommand{\eea}{\end{eqnarray}}
\newcommand{\E}{{\mathbb{E}}}

\newcommand{\lam}{{\lambda}} 
 
 \newcommand{\calL}{{\cal L}}

\newcommand{\lb}{{\langle}}
\newcommand{\rb}{{\rangle}}

\def\boldf#1{\hbox{\rlap{$#1$}\kern.4pt{$#1$}}}

\newcommand{\calG}{{\cal G}}
 
\newcommand{\rank}{{\rm rank}}

\ShortHeadings{New Perspectives on $k$-Support and Cluster Norms}{McDonald, Pontil, Stamos}
\firstpageno{1}

\begin{document}

\newtheorem{fact}[theorem]{Fact}

\title{New Perspectives on $k$-Support and Cluster Norms}

\author{\name Andrew M. McDonald \email a.mcdonald@cs.ucl.ac.uk\\
 \addr Department of Computer Science \\
       University College London \\
       Gower Street, London WC1E 6BT, UK        
\AND
		\name Massimiliano Pontil \email m.pontil@cs.ucl.ac.uk\\
		\addr 
		Istituto Italiano di Tecnologia \\ 
		Via Morego, 30, 16163 Genoa, Italy~ \\
		Department of Computer Science \\ University College London \\
       Gower Street, London WC1E 6BT, UK  
\AND 
       \name Dimitris Stamos \email d.stamos@cs.ucl.ac.uk \\
       \addr Department of Computer Science \\
       University College London \\
       Gower Street, London WC1E 6BT, UK
       }

\editor{}

\maketitle

\begin{abstract}
We study a regularizer which is defined as a parameterized infimum of quadratics, and which we call the box-norm.  
We show that the $k$-support norm, a regularizer proposed by \cite{Argyriou2012} for sparse vector prediction problems, belongs to this family, and the box-norm can be generated as a perturbation of the former.  
We derive an improved algorithm to compute the proximity operator of the squared box-norm, and we provide a method to compute the norm. 
We extend the norms to matrices, introducing the spectral $k$-support norm and spectral box-norm. 
We note that the spectral box-norm is essentially equivalent to the cluster norm, a multitask learning regularizer introduced by \citet{Jacob2009-CLUSTER}, and which in turn can be interpreted as a perturbation of the spectral $k$-support norm.  Centering the norm is important for multitask learning and we also provide a method to use centered versions of the norms as regularizers. 
Numerical experiments indicate that the spectral $k$-support and box-norms and their centered variants
provide state of the art performance in matrix completion and multitask learning problems respectively.

\end{abstract}

\vspace{.2truecm}

\section{Introduction}
\label{sec:intro}
We continue the study of a family of norms which are obtained by taking the infimum of a class of quadratic functions. 
These norms can be used as a regularizer in linear regression learning problems, where the parameter set can be tailored to assumptions on the underlying regression model. 
This family of norms is sufficiently rich to encompass regularizers such as the $\ell_p$ norms, the group Lasso with overlap \citep{Jacob2009-GL} and the norm of \cite{Micchelli2013}. 
In this paper we focus on a particular norm in this framework -- the box-norm -- in which the parameter set involves box constraints and a linear constraint. 
We study the norm in detail and show that it can be generated as a perturbation of the $k$-support norm introduced by \citet{Argyriou2012} for sparse vector estimation, which hence can be seen as a special case of the box-norm. 
Furthermore, our variational framework allows us to study efficient algorithms to compute the norms and the proximity operator of the square of the norms.

Another main goal of this paper is to extend the $k$-support and box-norms to a matrix setting. 
We observe that both norms are symmetric gauge functions, hence by applying them to the spectrum of a matrix we obtain two orthogonally invariant matrix norms. 
In addition, we observe that the spectral box-norm is essentially equivalent to the cluster norm introduced by \cite{Jacob2009-CLUSTER} for multitask clustering, which in turn can be interpreted as a perturbation of the spectral $k$-support norm.

The characteristic properties of the vector norms translate in a natural manner to matrices. In particular, the unit ball of spectral $k$-support norm is the convex hull of the set of matrices of rank 
no greater than $k$, and Frobenius norm bounded by one. 
In numerical experiments we present empirical evidence on the strong performance of the spectral $k$-support norm in low rank matrix completion and multitask learning problems.

Moreover, our computation of the vector box-norm and its proximity operator extends naturally to the spectral case, which allows us to use proximal gradient methods to solve regularization problems using the cluster norm. Finally, we provide a method to use the centered versions of the penalties, which are important in applications \citep[see e.g.][]{Evgeniou2007,Jacob2009-CLUSTER}.

\subsection{Related Work}
Our work builds upon a recent line of papers which considered convex regularizers defined as an infimum problem over a parametric family of quadratics, as well as related infimal convolution problems 
\citep[see][and references therein]{Jacob2009-GL,Bach2011,Maurer2012,Micchelli2005,Obozinski2012}.
Related variational formulations for the Lasso have also been discussed in \citep{Grandvalet1998} and further studied in \citep{Grandvalet2007}.

To our knowledge, the box-norm was first suggested by \cite{Jacob2009-CLUSTER} and used as a symmetric gauge function in matrix learning problems. 
The induced orthogonally invariant matrix norm is named the {\em cluster norm} in \citep{Jacob2009-CLUSTER} and was motivated as a convex relaxation of a multitask clustering problem. Here we formally prove that the cluster norm is indeed an orthogonal invariant norm. More importantly, we explicitly compute the norm and its proximity operator. 

A key observation of this paper is the link between the box-norm and the $k$-support norm and in turn the link between the cluster norm and the spectral $k$-support norm. 
The $k$-support norm was proposed in \citep{Argyriou2012} for sparse vector prediction and was shown to empirically outperform the Lasso \citep{Tibshirani1996} and Elastic Net \citep{Zou2005} penalties. 
See also \cite{Gkirtzou2013} for further empirical results. 

In recent years there has been a great deal of interest in the problem of learning a low rank matrix from a set of linear measurements. 
A widely studied and successful instance of this problem arises in the context of matrix completion or collaborative filtering, in which we want to recover a low rank (or approximately low rank) matrix from a small sample of its entries, see e.g. \citet{Srebro2005,Abernethy2009} and references therein. 
One prominent method of solving this problem is trace norm regularization: we look for a matrix which closely fits the observed entries and has a small trace norm (sum of singular values) \citep{Jaggi2010,Toh2011,Mazumder2010}. In our numerical experiments we consider the spectral $k$-support norm and spectral box-norm as alternatives to the trace norm and compare their performance.

Another application of matrix learning is multitask learning. 
In this framework a number of tasks, such as classifiers or regressors, are learned by taking advantage of commonalities between them.  This can improve upon learning the tasks separately, for instance when insufficient data is available to solve each task in isolation \citep[see e.g.][]{Evgeniou2005, Argyriou2006, Argyriou2008, Jacob2009-CLUSTER, Cavallanti, Maurer2006, Maurer2008}.
An approach which has been successful is the use of spectral regularizers such as the trace norm to learn a matrix where the columns represent the individual tasks, and in this paper we compare the performance of the spectral $k$-support and box-norms as penalties in multitask learning problems.  

Finally, we note that this is a longer version of the conference paper \citep{McDonald2014a} and includes new theoretical and experimental results.   


\subsection{Contributions}
We summarise the main contributions of this paper.
\begin{itemize}
\item 
We show that the vector $k$-support norm is a special case of the more general \emph{box-norm}, which in turn can be seen as a perturbation of the former. 
The box-norm can be written as a parameterized infimum of quadratics, and this framework is instrumental in deriving a fast algorithm to compute the norm and the proximity operator of the squared norm in $\mathcal{O}(d \log d)$ time. 
Apart from improving on the $\mathcal{O}(d(k + \log d))$ algorithm for the proximity operator in \citet{Argyriou2012}, this method allows one to use optimal first order optimization algorithms \citep{Nesterov2007} for the box-norm\footnote{We note that recently \cite{Chatterjee2014} showed that the proximity operator of the vector $k$-support norm can be computed in $O(d\log d)$. Here we directly follow \cite{Argyriou2012} and consider the squared $k$-support norm.}. 
\item 
We extend the $k$-support and box-norms to orthogonally invariant matrix norms. We note that the spectral box-norm is essentially equivalent to the cluster norm, which in turn can be interpreted as a perturbation of the spectral $k$-support norm in the sense of the Moreau envelope. 
Our computation of the vector box-norm and its proximity operator also extends naturally to the spectral case. 
This allows us to use proximal gradient methods for the cluster norm. 
Furthermore, we provide a method to apply the centered versions of the penalties, which are important in applications.
\item
We present extensive numerical experiments on both synthetic and real matrix learning datasets. 
Our findings indicate that regularization with the spectral $k$-support and box-norms produces state-of-the art results on a number of popular matrix completion benchmarks and centered variants of the norms show a significant improvement in performance over the centered trace norm and the matrix elastic net on multitask learning benchmarks. 
\end{itemize}

\subsection{Notation}
\label{sec:notation}

We use $\NN{n}$ for the set of integers from $1$ up to and including $n$. 
We let $\R^d$ be the $d$ dimensional real vector space, whose elements are denoted by lower case letters. We let $\R^d_+$ and $\R^d_{++}$ be the subsets of vectors with nonnegative and strictly positive components, respectively. 
We denote by $\Delta^{d}$ the unit $d$-simplex, $\Delta^{d} = \{\lambda \in \R^{d+1}: \sum_{i=1}^{d+1} \lambda_i = 1\}$. 
For any vector $w\in \R^d$, its {\em support} is defined as $\textrm{supp}(w) = \{i: w_i \neq 0\} \subseteq \NN{d}$. 
We use $1$ to denote either the scalar or a vector of all ones, whose dimension is determined by its context.  
Given a subset $g$ of $\NN{d}$, the $d$-dimensional vector $1_g$ has ones on the support $g$, and zeros elsewhere. 
We let $\R^{d \times T}$ be the space of $d \times T$ real matrices and write $W=[w_1, \ldots, w_T]$ to denote the matrix whose columns are formed by the vectors
$w_1, \ldots, w_T \in \R^d$. For a vector $\sigma \in \R^d$, we denote by $\textrm{diag}(\sigma)$ the $d \times d$ diagonal matrix having elements $\sigma_i$ on the diagonal. 
We say matrix $W \in \R^{d \times T}$ is diagonal if $W_{ij}=0$ whenever $i\ne j$. 
We denote the trace of a matrix $W$ by $\tr (W)$, and its rank by $\textrm{rank}(W)$. 
We let $\sigma(W) \in \R_+^r$ be the vector formed by the singular values of $W$, where $r=\min(d,T)$, and where we assume that the singular values are ordered nonincreasing, i.e. 
$\sigma_1(W) \geq \ldots \geq \sigma_r(W) \geq 0$.
We use ${\bf S}^d$ to denote the set of real $d \times d$ symmetric matrices, and ${\bf S}^d_+$ to denote the subset of positive semidefinite matrices. 
We use $\succeq$ to denote the positive semidefinite ordering on ${\bf S}^d$. 
The notation $\langle \cdot, \cdot \rangle$ denotes the standard inner products on $\R^d$ and $\R^{d \times T}$, that is $\langle x,y \rangle = \sum_{i=1}^d x_i y_i$ for $x,y \in \R^d$, and $\langle X, Y \rangle = \tr (X\trans Y)$, for $X,Y \in \R^{d \times T}$. 
Given a norm $\Vert \cdot \Vert$ on $\R^d$ or $\R^{d \times T}$, $\Vert \cdot \Vert_*$ denotes the corresponding dual norm, given by $\Vert u \Vert_* = \sup \{ \langle u,w \rangle : \Vert w \Vert \leq 1  \}$. 
On $\R^d$ we denote by $\Vert \cdot \Vert_2$ the Euclidean norm, and on $\R^{d \times T}$ we denote by $\Vert \cdot \Vert_{\rm F}$ the Frobenius norm and by $\Vert \cdot \Vert_{\rm tr}$ the trace norm, that is the sum of singular values.  

\subsection{Organization}
The paper is organized as follows. 
In Section \ref{sec:vector-norms}, we review a general class of norms and characterize their unit ball. 
In Section \ref{sec:box}, we specialize these norms to the box-norm, which we show is a perturbation of the $k$-support norm. 
We study the properties of the norms and we describe the geometry of the unit balls. 
In Section \ref{sec:computation-of-norm-and-prox}, we compute the box-norm and we provide an efficient method to compute the proximity operator of the squared norm. 
In Section \ref{sec:matrix-norms}, we extend the norms to orthogonally invariant matrix norms -- the spectral $k$-support and spectral box-norms -- 
and we show that these exhibit a number of properties which relate to the vector properties in a natural manner. 
In Section \ref{sec:MTL}, we review the clustered multitask learning setting, we recall the cluster norm introduced by \citet{Jacob2009-CLUSTER} and we show that the cluster norm corresponds to the spectral box-norm.  
We also provide a method for solving the resulting matrix regularization problem using ``centered" norms. 
In Section \ref{sec:numerics}, we apply the norms to matrix learning problems on a number of simulated and real datasets and report on their performance.  
In Section \ref{sec:extensions}, we discuss extensions to the framework and suggest directions for future research.  Finally, in Section \ref{sec:conclusion}, we conclude. 

\section{Preliminaries}
\label{sec:vector-norms}
In this section we review a family of norms parameterized by a set $\Theta$, and which we call the $\Theta$-norms. 
They are closely related to the norms considered in \citet{Micchelli2010,Micchelli2013}. Similar norms are also discussed in \citet[][Sect.~1.4.2]{Bach2011} where they are called $H$-norms.  We first recall the definition of the norm. 

\begin{definition}
\label{def:theta-norms}
Let $\Theta$ be a convex bounded subset of the open positive orthant.  For $w \in \R^d$ the $\Theta$-norm is defined as
\begin{align}
\Vert w \Vert_{\Theta} = 
\sqrt{ \inf_{\theta \in \Theta} \hspace{.05truecm} \sum_{i=1}^d \frac{w_i^2}{\theta_i} }. \label{eqn:theta-primal}
\end{align}
\end{definition}
Note that the function $(w,\theta) \mapsto  \sum_{i=1}^d \frac{w_i^2}{\theta_i}$ is strictly convex on $\R^d \times \R_{++}^d$, hence every minimizing sequence converges to the same point. 
The infimum is, however, not attained in general because a minimizing sequence may converge to a point on the boundary of $\Theta$. 
For instance, if $\Theta = \{\theta \in \R^d_{++}: \sum_{i=1}^d \theta_i \leq 1\}$, then $\|w\|_\Theta =\|w\|_1$ and the minimizing sequence converges to the point $(\frac{|w_1|~}{\Vert w \Vert_1},\dots,\frac{|w_d|~}{\Vert w \Vert_1})$, which belongs to $\Theta$ only if all the components of $w$ are different from zero.

\begin{proposition}
\label{prop:theta-is-norm}
The $\Theta$-norm is well defined and the dual norm is given, for $u \in \R^d$, by 
\begin{align}
\Vert u \Vert_{*,\Theta} = \sqrt{ \sup\limits_{\theta \in \Theta}  \hspace{.05truecm} \sum\limits_{i=1}^d \theta_i u_i^2} \label{eqn:theta-dual}.
\end{align}
\end{proposition}
\begin{proof}
Consider the expression for the dual norm. The function $\|\cdot\|_{*,\Theta}$ is a norm since it is a supremum of norms. Recall that the Fenchel conjugate $h^*$ of a function $h:\mathbb{R}^d\rightarrow\mathbb{R}$ is defined for every $u \in \R^d$ as $h^*(u) = \sup \left\{\langle u,w \rangle-h(w) : w \in \R^d\right\}$. 
It is a standard result from convex analysis that for any norm $\Vert \cdot \Vert$, the Fenchel conjugate of the function $h := \frac{1}{2}\Vert \cdot \Vert^2$ satisfies $h^* = \frac{1}{2}\Vert \cdot \Vert_*^2$, where $\Vert \cdot \Vert_*$ is the corresponding dual norm \citep[see, e.g.][]{Lewis1995}. By the same result, for any norm the biconjugate is equal to the norm, that is $(\Vert \cdot \Vert^{*})^* = \Vert \cdot \Vert$. 
Applying this to the dual norm we have, for every $w \in \R^d$, that
\begin{align*}
h(w) = \sup_{u\in \R^d } \left\{ \langle w,u \rangle - h^*(u)\right\} = 
 \sup_{u \in \R^d} \inf_{\theta \in \Theta }  \left\{ \sum_{i=1}^d \left( w_i u_i - \frac{1}{2} \theta_i u_i^2\right)  \right\}.	
\end{align*}
This is a minimax problem in the sense of von Neumann  \citep[see e.g. Prop. 2.6.3 in][]{Bertsekas2003}, and we can exchange the order of the $\inf$ and the $\sup$, and solve the latter (which is in fact a maximum) componentwise.  The gradient with respect to $u_i$ is zero for $u_i =\frac{w_i}{\theta_i}$, and substituting this into the objective function we obtain $h(w)= \frac{1}{2} \|w\|_\Theta^2$. It follows that the expression in \eqref{eqn:theta-primal} defines a norm, and its dual norm is defined by \eqref{eqn:theta-dual}, as required.
\end{proof}

The $\Theta$-norm \eqref{eqn:theta-primal} encompasses a number of well known norms.  
For instance, for $p\in [1,\infty)$ the $\ell_p$ norm is defined, for every $w \in \R^d$, as $\|w\|_p = \big(\sum_{i=1}^d |w_i|^p\big)^{\frac{1}{p}}$, if $p \in [1,\infty)$~and $\|w\|_\infty = \max_{i=1}^d |w_i|$. 
For $p \in [1,2)$, one can show \citep[][Lemma 26]{Micchelli2005}, that $\Vert w \Vert_p = \Vert w\Vert_{\Theta_p}$, where we have defined $\Theta_p= \big\{\theta \in \R^d_{++} : \sum_{i=1}^d \theta_i^{\frac{p}{2-p}} \leq 1\big\}$. For $p=1$ this confirms the set $\Theta$ corresponding to the $\ell_1$ norm as claimed above. Similarly, for $p\in (2, \infty]$ we have that $\Vert w \Vert_p = \Vert w \Vert_{*,\Theta_q}$, where $\frac{1}{p}+\frac{1}{q}=1$. The $\ell_2$-norm is obtained as both a primal and dual $\Theta$-norm in the limit as $p$ tends to 2. 
See also \citet{Aflalo2011} who considered the case of $p>2$.

Other norms which belong to the family \eqref{eqn:theta-primal} are presented in \citep{Micchelli2013} and correspond to choosing $\Theta =  \{\theta \in \Lambda : \sum_{i=1}^d \theta_i \leq 1\}$, where $\Lambda \subseteq \R_{++}^d$ is a convex cone. 
A specific example described therein is the wedge penalty, which corresponds to choosing $\Lambda = \{\theta \in \R^d_{++},~ \theta_1 \geq \ldots \geq \theta_d\}$.

We now describe the unit ball of the $\Theta$-norm when the set $\Theta$ is a polyhedron and we characterize the unit ball of the norm.  This setting applies to a number of norms of practical interest, including the group lasso with overlap, the wedge norm mentioned above and, as we shall see, the $k$-support norm. To describe our observation, for every vector $\gamma \in \R_{+}^d$, we define the seminorm
\begin{align}
\|w\|_{\gamma} = \sqrt{\sum_{i : \gamma_i > 0} \frac{w_i^2}{\gamma_i}}.\notag 
\end{align}
\begin{proposition}
\label{prop:geo2}
Let $\gamma^1,\dots,\gamma^m \in \R_+^d$ such that $\sum_{\ell=1}^m \gamma^\ell \in \R_{++}^d$ and let $\Theta = \{\theta \in \R_{++}^d : \theta = \sum_{\ell=1}^m \lambda_\ell \gamma^\ell,~\lam \in \Delta^{m-1}\}$. 

Then we have, for every $w\in \R^d$, that
\begin{align}
\label{Alabel}
\|w\|_\Theta = \inf \left\{ \sum_{\ind=1}^m 
\|v_\ind\|_{\gamma^\ind} :  v_\ind \in \R^d,~{\rm supp}(v_\ind) \subseteq \supp(\gamma^\ind),~\ell \in \NN{m},~\sum_{\ind=1}^m v_\ind = w \right \}.
\end{align}
Moreover, the unit ball of the norm is given by the convex hull of the set
\begin{align}
 \bigcup\limits_{\ind=1}^m 
 \left\{w\in \R^d: \supp(w) \subseteq \supp(\gamma^\ind),
 \|w\|_{\gamma^\ind} \leq 1\right\}.
 \label{Blabel}
\end{align}
\end{proposition}
The proof of this result is presented in the appendix. It is based on observing that the Minkowski functional \citep[see e.g.][]{Rudin1991} of the convex hull of the set \eqref{Blabel} is a norm and it is given by the right hand side of equation \eqref{Alabel}; we then prove that this norm coincides with $\|\cdot\|_\Theta$ by noting that both norms share the same dual norm. To illustrate an application of the proposition, we specialize it to the group Lasso with overlap \citep{Jacob2009-GL}.
\begin{corollary}
If $\calG$ is a collection of subsets of $\NN{d}$ such that $\bigcup_{g \in \calG} g = \NN{d}$ and $\Theta$ is the interior of the set ${\rm co}\{1_g: {g\in \G} \}$, then we have, for every $w\in \R^d$, that
\begin{align}
\label{Alabel2}
\|w\|_\Theta = \inf \left\{ \sum_{g \in \G}
\|v_g\|_2 :  v_g \in \R^d,~{\rm supp}(v_g) \subseteq g,~\sum_{g\in \G} v_g = w \right \}.
\end{align}
Moreover, the unit ball of the norm is given by the convex hull of the set
\begin{align}
 \bigcup\limits_{g \in \G}
 \left\{w\in \R^d: \supp(w) \subseteq g,
 \|w\|_2 \leq 1\right\}.
 \label{Blabel2}
\end{align}
\label{cor:GLO}
\end{corollary}
We do not claim any originality in the above corollary and proposition, although we cannot find a specific reference. The utility of the result is that it links seemingly different norms such as the group Lasso with overlap and the $\Theta$-norms, which provide a more compact representation, involving only $d$ additional variables. This formulation is especially useful whenever the optimization problem \eqref{eqn:theta-primal} can be solved in closed form.  One such example is provided by the wedge norm described above. In the next section we discuss one more important case, the box-norm, which plays a central role in this paper.

\section{The Box-Norm and the $k$-Support Norm}
\label{sec:box}
We now specialize our analysis to the case that
\begin{align}
\Theta = \left\{ \theta \in \mathbb{R}^d: a \leq \theta_i \leq b, \sum_{i=1}^d \theta_i \leq c\right\} \label{eqn:box-theta-def}
\end{align} where $0 < a \leq b$ and $c \in [ad,bd]$. We call the norm defined by \eqref{eqn:theta-primal} the \emph{box-norm} and we denote it by $\|\cdot\|_{\rm box}$.  

The structure of set $\Theta$ for the box-norm will be fundamental in computing the norm and deriving the proximity operator in Section \ref{sec:computation-of-norm-and-prox}.
Furthermore, we note that the constraints are invariant with respect to permutations of the components of $\Theta$ and, as we shall see in Section \ref{sec:matrix-norms}, this property is key to extending the norm to matrices.  
Finally, while a restriction of the general family, the box-norm nevertheless encompasses a number of norms including the $\ell_1$ and $\ell_2$ norms, as well as the $k$-support norm, which we now recall.

For every $k \in \N_d$, the $k$-support norm $\|\cdot\|_{(k)}$ \citep{Argyriou2012} is defined as the norm whose unit ball is the convex hull of the set of vectors of cardinality at most $k$ and $\ell_2$-norm no greater than one. The authors show that the $k$-support norm can be written as the infimal convolution \citep[see][p.~34]{Rockafellar1970}
\begin{align}
\Vert w \Vert_{(k)} &=  \inf  \left\{ \sum_{g \in \G_k} \Vert v_g \Vert_2 : v_g \in \R^d,~{\rm supp}(v_g) \subseteq g,~\sum_{g \in \G_k} v_g = w \right\},~~~w \in \R^d,
\label{eqn:GLO}
\end{align}
where 
$\mathcal{G}_k$ is the collection of all subsets of $\NN{d}$ containing at most $k$ elements. The $k$-support norm is a special case of the group lasso with overlap \citep{Jacob2009-GL}, where the cardinality of the support sets is at most $k$.  
When used as a regularizer, the norm encourages vectors $w$ to be a sum of a limited number of vectors with small support.  
Note that while definition \eqref{eqn:GLO} involves a combinatorial number of variables, \citet{Argyriou2012} observed that the norm can be computed in $\mathcal{O}(d \log d)$, a point we return to in Section \ref{sec:computation-of-norm-and-prox}.

Comparing equation \eqref{eqn:GLO} with Corollary \ref{cor:GLO} it is evident that the $k$-support norm is a $\Theta$-norm where $\Theta = \{\theta \in \R_{++}^d : \theta = \sum_{g \in \G_k} \lam_g 1_{g},~\lam \in \Delta^{|\G_k|-1}\}$, which  
by symmetry can be expressed as $\Theta = \{\theta : 0<\theta_i \leq 1, \sum_{i=1}^d \theta_i \leq k\}$. 
Hence, we see that the $k$-support norm is a special case of the box-norm. 

Despite the complicated form of \eqref{eqn:GLO}, \cite{Argyriou2012} observe that the dual norm has a simple formulation, 
namely the $\ell_2$-norm of the $k$ largest components,
\begin{align}
\Vert u \Vert_{*,(k)} &= \sqrt{\sum_{i=1}^k (\vert u \vert^{\downarrow}_i)^2},~~~u \in \R^d\label{eqn:ksup-dualeq},
\end{align} 
where $|u|^{\downarrow}$ is the vector obtained from $u$ by reordering its components so that they are non-increasing in absolute value. 
Note from equation \eqref{eqn:ksup-dualeq} that for $k=1$ and $k=d$, the dual norm is equal to the $\ell_{\infty}$-norm and $\ell_2$-norm, respectively. 
It follows that the $k$-support norm includes the $\ell_1$-norm and $\ell_2$-norm as special cases. 

We now provide a different argument illustrating that the $k$-support norm belongs to the family of box-norms using the dual norm. We first derive the dual box-norm.
\begin{proposition}\label{prop:dual-of-theta}
The dual box-norm is given by
\begin{align}
\Vert u \Vert_{*,{\rm box}}^2 = a \|u\|_2^2 + (b-a) \left(\Vert u \Vert_{*,(k)}^2 + (\rho - k)(\vert u\vert^{\downarrow}_{k+1})^2\right),
\label{eq:dual-abc}
\end{align}
where $\rho = \frac{c-da}{b-a}$ and $k$ is the largest integer not exceeding $\rho$. 
\end{proposition}
\begin{proof}
We need to solve problem \eqref{eqn:theta-dual}. We make the change of variable $\phi_i = \frac{\theta_i - a}{b-a}$ and observe that the constraints on $\theta$ induce the constraint set $\big\{\phi \in (0,1]^d,~\sum_{i=1}^d \phi_i \leq \rho\big\}$, where $\rho = \frac{c - da}{b-a}$. Furthermore $\sum_{i=1}^d \theta_i u_i^2 = a \|u\|_2^2 + (b-a) \sum_{i=1}^d \phi_i u_i^2$.
The result then follows by taking the supremum over $\phi$.
\end{proof}
We see from equation \eqref{eq:dual-abc} that the dual norm decomposes into a weighted combination of the $\ell_2$-norm, the $k$-support norm and a residual term, which {vanishes} if $\rho = k \in \N_d$. For the rest of this paper we assume this holds, which loses little generality. This choice is equivalent to requiring that $c = (b-a)k + da$, which is the case considered by \cite{Jacob2009-CLUSTER} in the context of multitask clustering, where $k+1$ is interpreted as the number of clusters and $d$ as the number of tasks. We return to this case in Section \ref{sec:MTL}, where we explain in detail the link between the spectral $k$-support norm and the cluster norm.

Observe that if $a=0, b=1$, and $\rho=k$, the dual box-norm \eqref{eq:dual-abc} coincides with dual $k$-support norm in equation \eqref{eqn:ksup-dualeq}. We conclude that if 
$$\Theta  = \left\{\theta \in {\mathbb R}^d: 0 < \theta_i \leq 1,~\sum\limits_{i=1}^d \theta_i \leq k\right\}
$$ 
then the $\Theta$-norm coincides with the $k$-support norm.

\subsection{Properties of the Norms}
\label{sec:properties-of-vector-norms}
In this section we illustrate a number of properties of the box-norm and the connection to the $k$-support norm.  
The first result follows as a special case of Proposition \ref{prop:geo2}. 
\begin{corollary}
\label{cor:box-inf-convolution-unit-ball-dual-unit-ball}
If $0<a<b$ and $c= (b-a)k + da$, for $k \in \NN{d}$, then it holds that
\begin{align*}
\|w\|_{\rm box} = \inf \left\{
\sum_{g \in \G_k} 
\sqrt{ \hspace{-.05truecm}\sum_{i \in g} \hspace{-.015truecm}\frac{v_{g,i}^2}{b} \hspace{-.015truecm}+\hspace{-.015truecm}\sum_{i \notin g}   \hspace{-.015truecm}\frac{v_{g,i}^2}{a}} :v_g \in \R^d,~\sum\limits_{g \in \G_k} v_g = w
\hspace{-.03truecm}\right\},~~~w\in \R^d.
\end{align*}
Furthermore, the unit ball of the norm is given by the convex hull of the set
\begin{align}
\bigcup\limits_{g \in \G_k} \left\{w \in \R^d: \sum_{i \in g} \frac{w_i^2}{b} + \sum_{i\notin g} \frac{w_i^2}{a} \leq 1 
\right\} \label{eqn:box-norm-C}.
\end{align}
\end{corollary}
Notice in Equation \eqref{eqn:box-norm-C} that if $b=1$, then as $a$ tends to zero, we obtain the expression of the $k$-support norm \eqref{eqn:GLO},  
recovering in particular the support constraints.
If $a$ is small and positive, the support constraints are not imposed, however most of the weight for each $v_g$ tends to be concentrated on $\textrm{supp}(g)$.
Hence, Corollary \ref{cor:box-inf-convolution-unit-ball-dual-unit-ball} suggests that if $a \ll b$ then the box-norm regularizer will encourage vectors $w$ whose dominant components are a subset of a union of a small number of groups $g \in \G_k$.

Our next result links two $\Theta$-norms whose parameter sets are related by a linear transformation with positive coefficients. 
\begin{lemma}\label{lem:linear-transformation-Moreau}
Let $\Theta$ be a convex bounded subset of the positive orthant in $\R^d$, and let $\Phi = \{\phi \in \R^d: \phi_i = \alpha + \beta \theta_i , \theta \in \Theta\}$, where $\alpha, \beta >0$.  
Then
\begin{align*}
\Vert w\Vert_{\Phi}^2 &= \min_{z \in \R^d} \left\{ \frac{1}{\alpha} \Vert w-z \Vert_2^2 + \frac{1}{\beta} \Vert z \Vert_{\Theta}^2 \right\}.
\end{align*}
\end{lemma}

\begin{proof}
We consider the definition of the norm $\Vert \cdot \Vert_{\Phi}$ in \eqref{eqn:theta-primal}. 
We have
\begin{align}
\Vert w \Vert_{\Phi}^2 
&= \inf_{\phi \in \Phi} \sum_{i=1}^d \frac{w_i^2}{\phi_i}
=  \inf_{\theta \in \Theta} \sum_{i=1}^d \frac{w_i^2}{\alpha + \beta \theta_i } \label{eqn:phi-reg-alpha-beta},
\end{align}
where we have made the change of variable $\phi_i = \alpha + \beta \theta_i$. 
Next we observe that 
\begin{align}
\min_{z\in \R^d} \left\{ \frac{1}{\alpha}\Vert w - z \Vert_2^2 + \frac{1}{\beta} \Vert z \Vert_{\Theta}^2 \right\}
&=  \min_{z \in \R^d}  \inf_{\theta \in \Theta} \left\{ \sum_{i=1}^d \frac{(w_i-z_i)^2}{\alpha}  + \frac{z_i^2}{\beta \theta_i} \right\}
= \inf_{\theta \in \Theta} \sum_{i=1}^d \frac{w_i^2}{\alpha + \beta \theta_i} \label{eqn:theta-reg-alpha-beta},
\end{align}
where we interchanged the order of the minimum and the infimum and solved for $z$ componentwise, setting $z_i = \frac{\beta \theta_i w_i}{\alpha + \beta \theta_i}$. 
The result now follows by combining equations \eqref{eqn:phi-reg-alpha-beta} and \eqref{eqn:theta-reg-alpha-beta}.
\end{proof}

In Section \ref{sec:box} we characterized the $k$-support norm as a special case of the box-norm.  
Conversely, Lemma \ref{lem:linear-transformation-Moreau} allows us to interpret the box-norm as a perturbation of the $k$-support norm with a quadratic regularization term.

\begin{proposition}
\label{cor:box_as_k_sup_and_l2}
Let $\Vert \cdot \Vert_{\rm box}$ be the box-norm on $\mathbb{R}^d$ with parameters $0<a<b$ and $c= k(b-a) + da$, for $k \in \NN{d}$, then
\begin{align}
\Vert w \Vert_{\rm box}^2 = \min_{z \in \mathbb{R}^d} \left\{\frac{1}{a} \Vert w - z \Vert_2^2 + \frac{1}{b-a} \Vert z \Vert_{(k)}^2\right\}. \label{eq:keyic}
\end{align}
\end{proposition}

\begin{proof} 
The result directly follows from Lemma \ref{lem:linear-transformation-Moreau} for 
$\Theta = \{ \theta \in \R^d : 0 < \theta_i \leq 1, \sum_{i=1}^d \theta_i \leq k \}$, $\alpha = a$ and $\beta = b-a$.
\end{proof}

Lemma \ref{lem:linear-transformation-Moreau} and Proposition \ref{cor:box_as_k_sup_and_l2} can further be interpreted using the Moreau envelope from convex optimization, which we now recall \citep[Ch. 1 \S G]{Rockafellar2009}. 
\begin{definition}
Let $f: \R^d \rightarrow (-\infty, \infty]$ be proper, lower semi-continuous and let $\rho >0$.  The \emph{Moreau envelope} of $f$ with parameter $\rho$ is defined as
\begin{align*}
e_{\rho} f(w) = \inf_{z \in \R^d} \left\{ f(z) + \frac{1}{2 \rho} \Vert w-z\Vert_2^2 \right\}. 
\end{align*}
\end{definition}

Note that $e_{\rho}f$ minorizes $f$ and is convex and smooth \citep[see e.g.]{Bauschke2010}. 
It acts as a parameterized smooth approximation to $f$ from below, which motivates its use in variational analysis \citep[see e.g.][for further discussion]{Rockafellar2009}.
Lemma \ref{lem:linear-transformation-Moreau} therefore says that 
$\beta\Vert \cdot \Vert_{\Phi}^2$ 
is a Moreau-envelope of $\Vert \cdot \Vert_{\Theta}^2$ with parameter $\rho = \frac{\alpha}{2 \beta}$	whenever $\Phi$ is defined as $\Phi = \alpha + \beta \Theta$, $\alpha, \beta >0$. 
In particular we see from \eqref{eq:keyic} that the squared box-norm, scaled by a factor of $(b-a)$, is a Moreau envelope of the squared $k$-support norm as we have  
\begin{align}
(b-a) \Vert w \Vert_{\rm box}^2 
&= \min_{z \in \mathbb{R}^d} \left\{ \Vert z \Vert_{(k)}^2 + \frac{1}{2\rho} \Vert w - z \Vert_2^2  \right\} =:  
e_{\rho}f(w), \label{eqn:ksup-moreau}
\end{align}
where $f(w) = \Vert w \Vert_{(k)}^2$ and $\rho = \frac{1}{2}\frac{a}{b-a}$. 

Proposition \ref{cor:box_as_k_sup_and_l2} further allows us to decompose the solution to a vector learning problem using the squared box-norm into two components with particular structure.
Specifically, consider the regularization problem 
\begin{align}
\min_{w\in \R^d} \Vert Xw -y \Vert_2^2 + \lambda \Vert w \Vert_{\rm box}^2 \label{eqn:vector-regression}
\end{align}
with data $X \in \R^{n \times d}$ and response $y \in \R^n$. 
Using Proposition \ref{cor:box_as_k_sup_and_l2} and setting $w=u+z$, we see that \eqref{eqn:vector-regression} is equivalent to
\begin{align}
\min_{u,z\in \mathbb{R}^d} \left\{\Vert X(u+z)-y\Vert_2^2 + \frac{\lambda}{a} \Vert u \Vert_2^2 + \frac{\lambda}{b-a} \Vert z \Vert_{(k)}^2 \right\}. \label{eqn:dirty-regression}
\end{align}
Furthermore, if $(\hat{u},\hat{z})$ solves problem \eqref{eqn:dirty-regression} then $\hat{w}=\hat{u}+\hat{z}$ solves problem \eqref{eqn:vector-regression}. The solution ${\hat w}$ can therefore be interpreted as the superposition of a vector which has small $\ell_2$ norm, and a vector which has small $k$-support norm, with the parameter $a$ regulating these two components.  Specifically, as $a$ tends to zero, in order to prevent the objective from blowing up, ${\hat u}$ must also tend to zero and we recover $k$-support norm regularization. 
Similarly, as $a$ tends to $b$, ${\hat z}$ vanishes and we have a simple ridge regression problem.

A further consequence of Proposition \ref{cor:box_as_k_sup_and_l2} is the differentiability of the squared box-norm. 
\begin{proposition}\label{prop:abc-differentiable}
If $a >0$ the squared box-norm is differentiable on $\R^d$ and its gradient 
\begin{align*}
\nabla( \Vert \cdot \Vert_{\rm box}^2 ) &= \frac{2}{a} \left( \textrm{Id} - \prox_{\rho \Vert \cdot \Vert_{(k)}^2}  \right)
\end{align*}
is Lipschitz continuous with parameter $\frac{2}{a}$.
\end{proposition}
\begin{proof}
Letting $f(w) = \Vert w \Vert_{(k)}^2$, $\rho = \frac{1}{2}\frac{a}{b-a}$, by \eqref{eqn:ksup-moreau} we have $e_{\rho}f(w) = (b-a) \Vert w \Vert_{\rm box}^2$. 
The result follows directly from \citet[][Prop.~12.29]{Bauschke2010}, as $f$ is convex and continuous on $\R^d$ and the gradient is given as $\nabla(e_{\rho}f ) = \frac{1}{\rho} (\textrm{Id} - \prox_{\rho f})$.
\end{proof}

Proposition \ref{prop:abc-differentiable} establishes that the square of the box-norm is differentiable and its smoothness is controlled by the parameter $a$. Furthermore, the gradient can be determined from the proximity operator, which we compute in Section \ref{sec:computation-of-norm-and-prox}.  

\subsection{Geometry of the Norms}\label{sec:geometry}
In this section, we briefly investigate the geometry of the box-norm.  
Figure \ref{fig:3d-ksup-unit-balls} depicts the unit balls for the $k$-support norm in $\R^3$ for various parameter values, setting $b=1$ throughout. For $k=1$ and $k=3$ we recognize the $\ell_1$ and $\ell_{2}$ balls respectively. 
For $k=2$ the unit ball retains characteristics of both norms, and in particular we note the discontinuities along each of $x$, $y$ and $z$ planes, as in the case of the $\ell_1$ norm.

Figure \ref{fig:3d-box-1-unit-balls} depicts the unit balls for the box-norm for a range of values of $a$ and $k$, with $c=(b-a)k+da$.  
We see that in general the balls increase in volume with each of $a$ and $k$, holding the other parameter fixed. 
Comparing the $k$-support norm ($k=1$), that is the $\ell_1$ norm, and the box-norm ($k=1$, $a=0.15$), we see that the parameter $a$ smooths out the 
{sharp edges} of the $\ell_1$ norm.
This is also visible when comparing the $k$-support ($k=2$) and the box ($k=2$, $a=0.15$).  This illustrates the smoothing effect of the parameter $a$, as suggested by Proposition \ref{prop:abc-differentiable}. 

We can gain further insight into the shape of the unit balls of the box-norm from Corollary \ref{cor:box-inf-convolution-unit-ball-dual-unit-ball}. 
Equation \eqref{eqn:box-norm-C} shows that the primal unit ball is the convex hull of ellipsoids in $\R^d$, where for each group $g$ the semi-principle axis along dimension $i$ has length $\sqrt{b}$ if $i\in g$, and length $\sqrt{a}$ if $i\notin g$.  
Similarly, the unit ball of the dual box-norm is the intersection of ellipsoids in $\R^d$ where for each group $g$ the semi-principle axis in dimension $i$ has length $1/ \sqrt{b}$ if $i\in g$, and length $1/\sqrt{a}$ if $i\notin g$ (see also Equation \ref{eqn:dual-unit-ball} in the appendix). 
It is instructive to further consider the effect of the parameter $a$ on the unit balls for fixed $k$. 
To this end, recall that since $c=(b-a)k + da$, when $k=d$ we have $c=bd$. 
In this case, for all values of $a$ in $(0,b]$, the objective in \eqref{eqn:theta-primal} is attained by setting $\theta_i=b$ for all $i$, and we recover the $\ell_2$-norm, scaled by $1/\sqrt{b}$, for the primal box-norm. 
Similarly in \eqref{eqn:theta-dual}, the dual norm gives rise to the $\ell_2$-norm, scaled by $\sqrt{b}$. 
In the remainder of this section we therefore only consider the cases $k\in \{1,2\}$ in $\R^3$

For $k=1$, $\mathcal{G}_k = \{\{1\},\{2\}, \{3\}\}$. 
The unit ball of the primal box-norm is the convex hull of the ellipsoids defined by
\begin{align}
\frac{w_1^2}{b} + \frac{w_2^2}{a}+\frac{w_3^2}{a} =1,  
\quad \frac{w_1^2}{a} + \frac{w_2^2}{b}+\frac{w_3^2}{a} =1, \quad \text{and}
\quad \frac{w_1^2}{a} + \frac{w_2^2}{a}+\frac{w_3^2}{b} =1, \label{eqn:ellipses-primal-k1}
\end{align} 
and the unit ball of the dual box-norm is the intersection of the ellipsoids defined by
\begin{align}
\frac{w_1^2}{b^{-1}} + \frac{w_2^2}{a^{-1}}+\frac{w_3^2}{a^{-1}} =1,  
\quad \frac{w_1^2}{a^{-1}} + \frac{w_2^2}{b^{-1}}+\frac{w_3^2}{a^{-1}} =1, \quad \text{and}
\quad \frac{w_1^2}{a^{-1}} + \frac{w_2^2}{a^{-1}}+\frac{w_3^2}{b^{-1}} =1. \label{eqn:ellipses-dual-k1}
\end{align} 
For $k=2$, $\mathcal{G}_k =
\{\{1\},\{2\}, \{3\}, \{1,2\}, \{2,3\}, \{1,3\}\}$.  The unit ball of the primal box-norm is the convex hull of the ellipsoids defined by \eqref{eqn:ellipses-primal-k1} in addition to the following
\begin{align}
\frac{w_1^2}{b} + \frac{w_2^2}{b}+\frac{w_3^2}{a} =1,  
\quad \frac{w_1^2}{a} + \frac{w_2^2}{b}+\frac{w_3^2}{b} =1, \quad \text{and}
\quad \frac{w_1^2}{b} + \frac{w_2^2}{a}+\frac{w_3^2}{b} =1, \label{eqn:ellipses-primal-k2}
\end{align} 
and the unit ball of the dual box-norm is the intersection of the ellipsoids defined by \eqref{eqn:ellipses-dual-k1} in addition to the following
\begin{align}
\frac{w_1^2}{b^{-1}} + \frac{w_2^2}{b^{-1}}+\frac{w_3^2}{a^{-1}} =1,  
\quad \frac{w_1^2}{a^{-1}} + \frac{w_2^2}{b^{-1}}+\frac{w_3^2}{b^{-1}} =1, \quad \text{and}
\quad \frac{w_1^2}{b^{-1}} + \frac{w_2^2}{a^{-1}}+\frac{w_3^2}{b^{-1}} =1. \label{eqn:ellipses-dual-k2}
\end{align} 
For the primal norm, note that since $b>a$, each of the ellipsoids in \eqref{eqn:ellipses-primal-k1} is entirely contained within one of those defined by \eqref{eqn:ellipses-primal-k2}, hence when taking the convex hull we need only consider the latter set.
Similarly for the dual norm, since $\frac{1}{b}<\frac{1}{a}$, each of the ellipsoids in 
{\eqref{eqn:ellipses-dual-k1}} is contained within one of those defined by \eqref{eqn:ellipses-dual-k2}, hence when taking the intersection we need only consider the latter set. 

Figures \ref{fig:3dbox-k1-primal-unit-balls} and \ref{fig:3dbox-k2-dual-unit-balls} depict the constituent ellipses for various parameter values for the primal and dual norms. 
As $a$ tends to zero the ellipses become degenerate. 
For $k=1$, taking the convex hull we recover the $\ell_1$ unit ball in the primal norm, and taking the intersection we recover the $\ell_{\infty}$ unit ball in the dual norm. 
As $a$ tends to $1$ we recover the $\ell_2$ norm in both the primal and the dual.

\begin{figure}[th!]
\caption{Unit balls of the $k$-support norm for $k\in \{1,2,3\}$.\label{fig:3d-ksup-unit-balls}}
\centering
\begin{minipage}{0.30\linewidth}
\centering
  \includegraphics[width=0.8\linewidth]{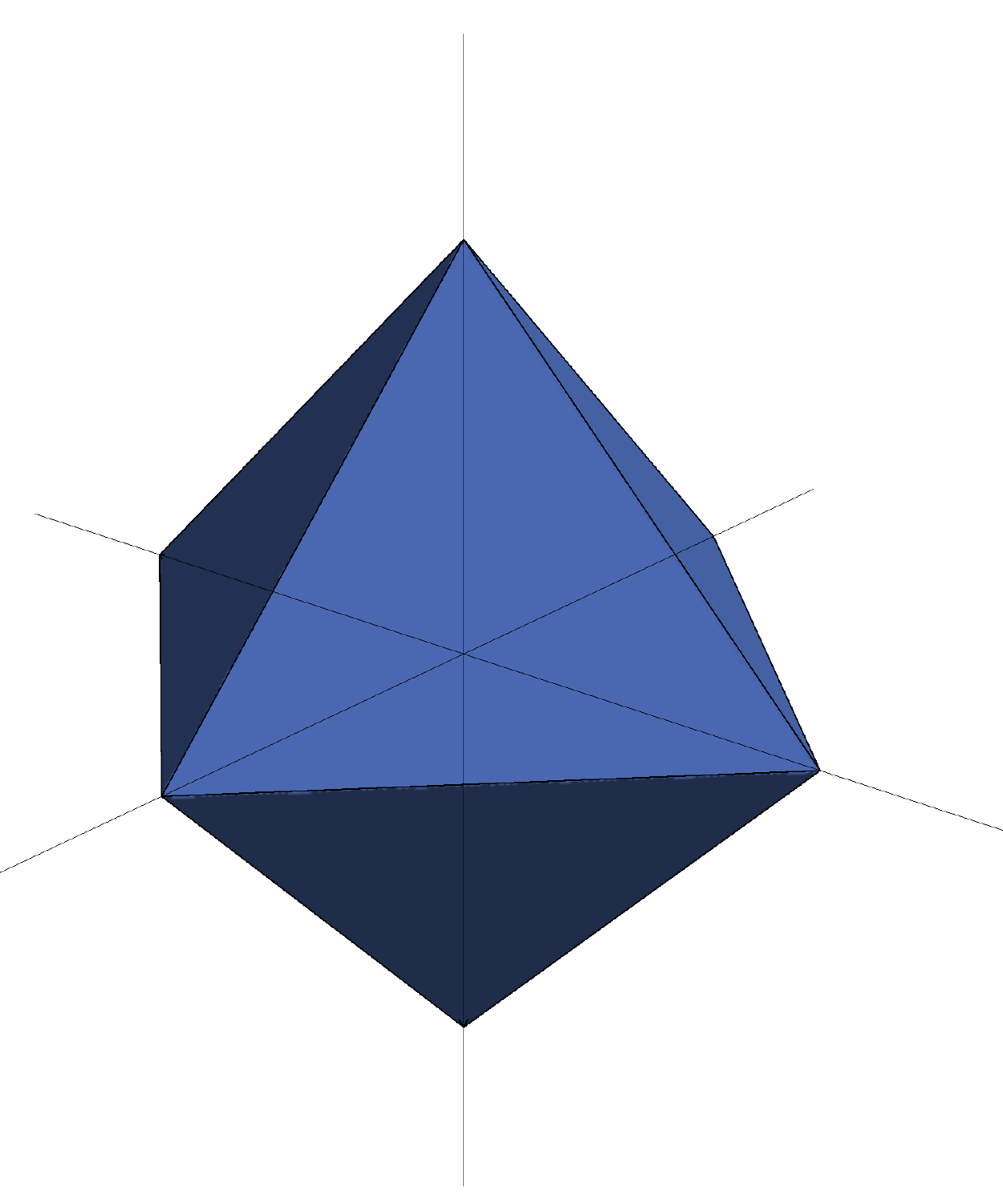}
\end{minipage}
\begin{minipage}{0.30\linewidth}
\centering
  \includegraphics[width=0.8\linewidth]{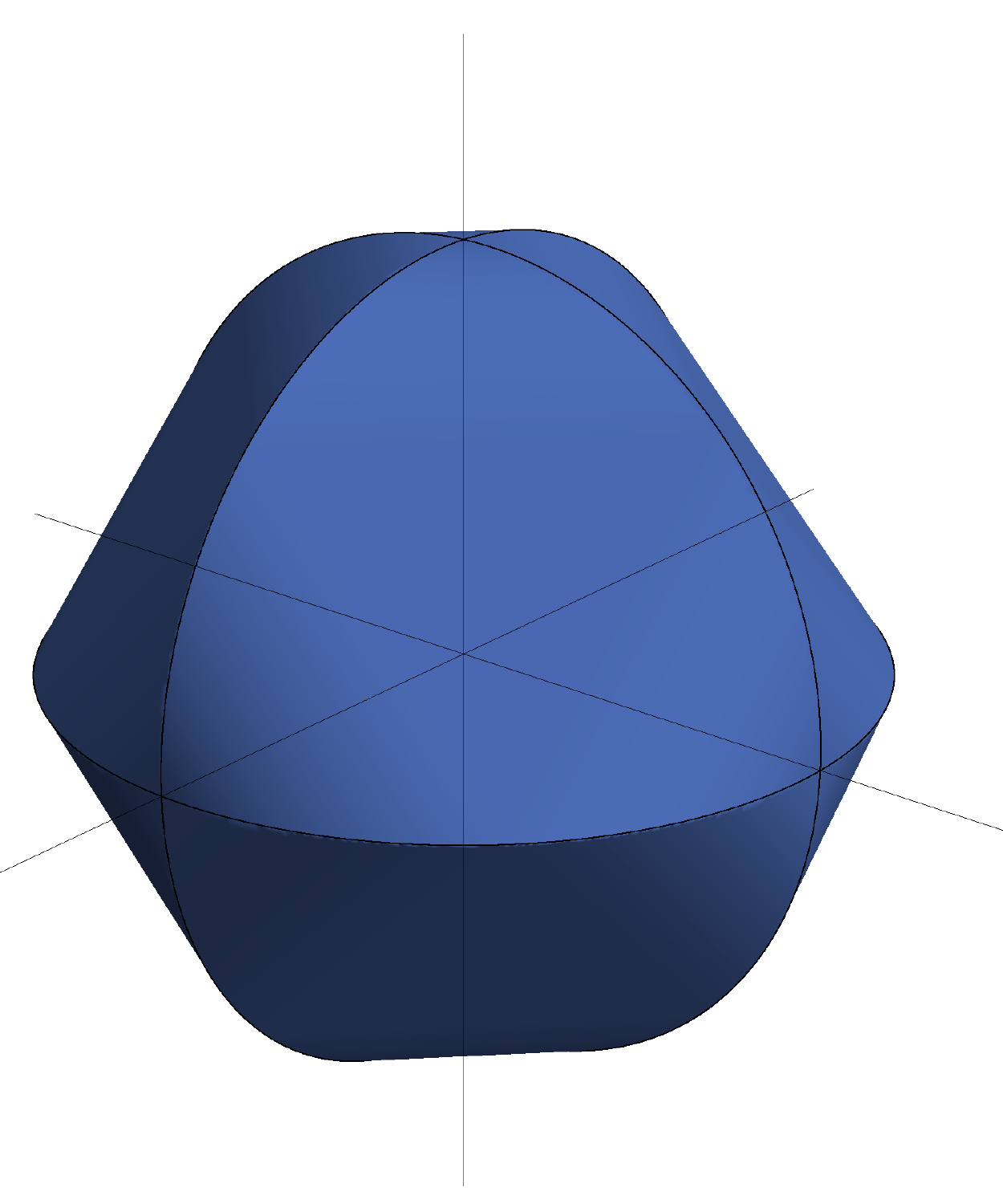}
\end{minipage}
\begin{minipage}{0.30\linewidth}
\centering
  \includegraphics[width=0.8\linewidth]{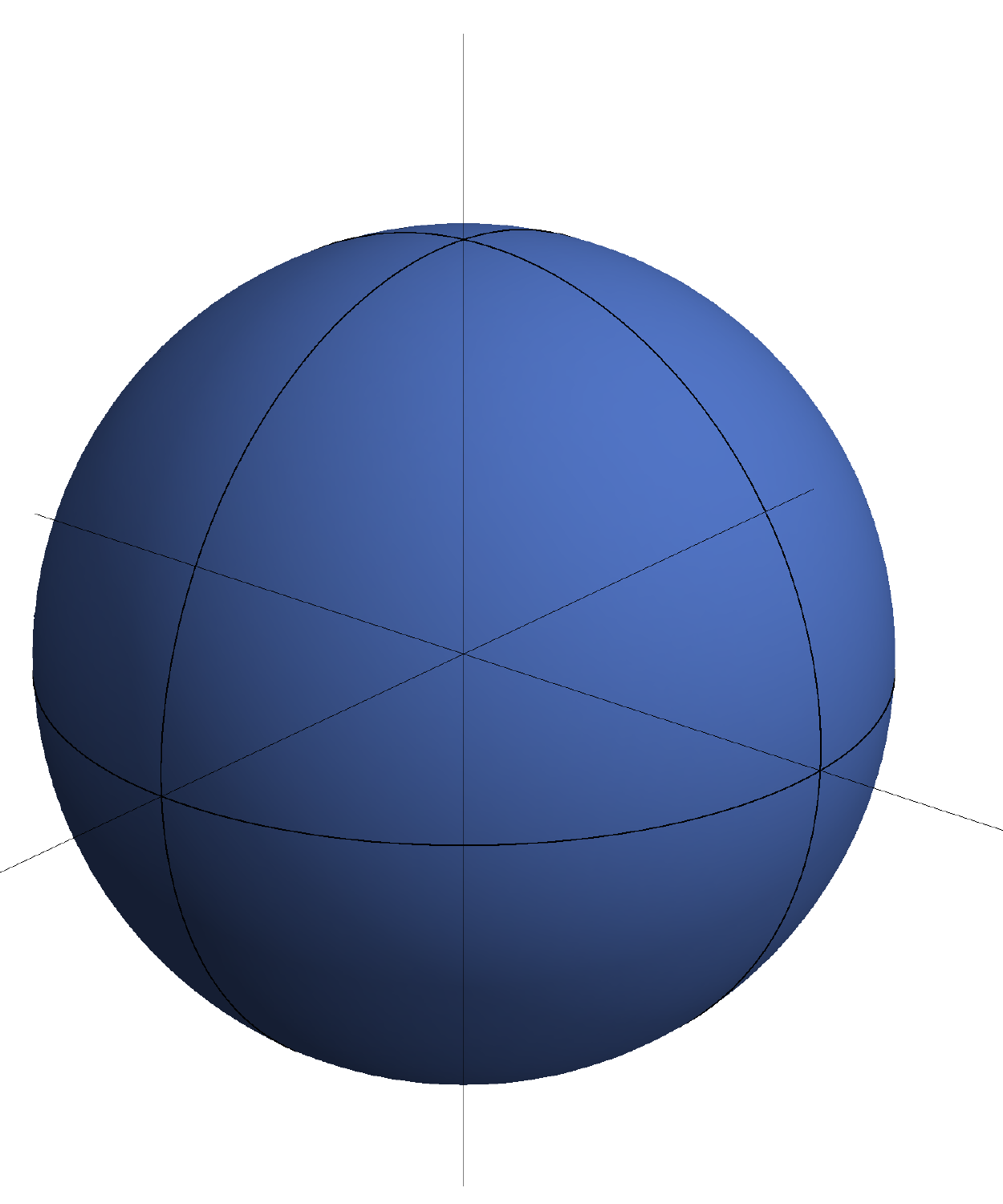}
\end{minipage}
\caption{Unit balls of the box-norm, $(k,a)\in \{(1,0.15), (2,0.15),(2,0.40)\}$.\label{fig:3d-box-1-unit-balls}}
\centering
\begin{minipage}{0.30\linewidth}
\centering
  \includegraphics[width=0.8\linewidth]{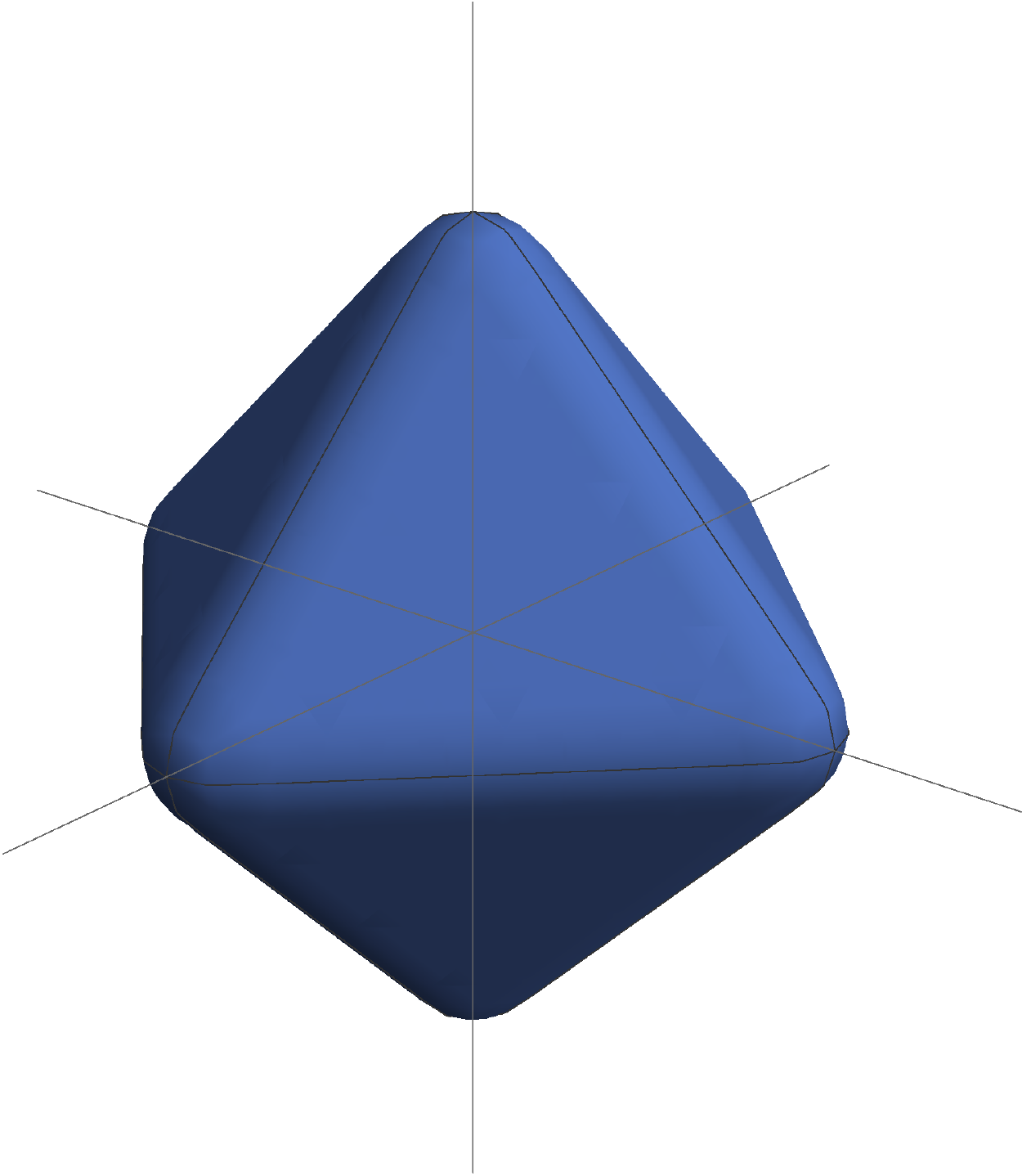}
\end{minipage}
\begin{minipage}{0.30\linewidth}
\centering
  \includegraphics[width=0.8\linewidth]{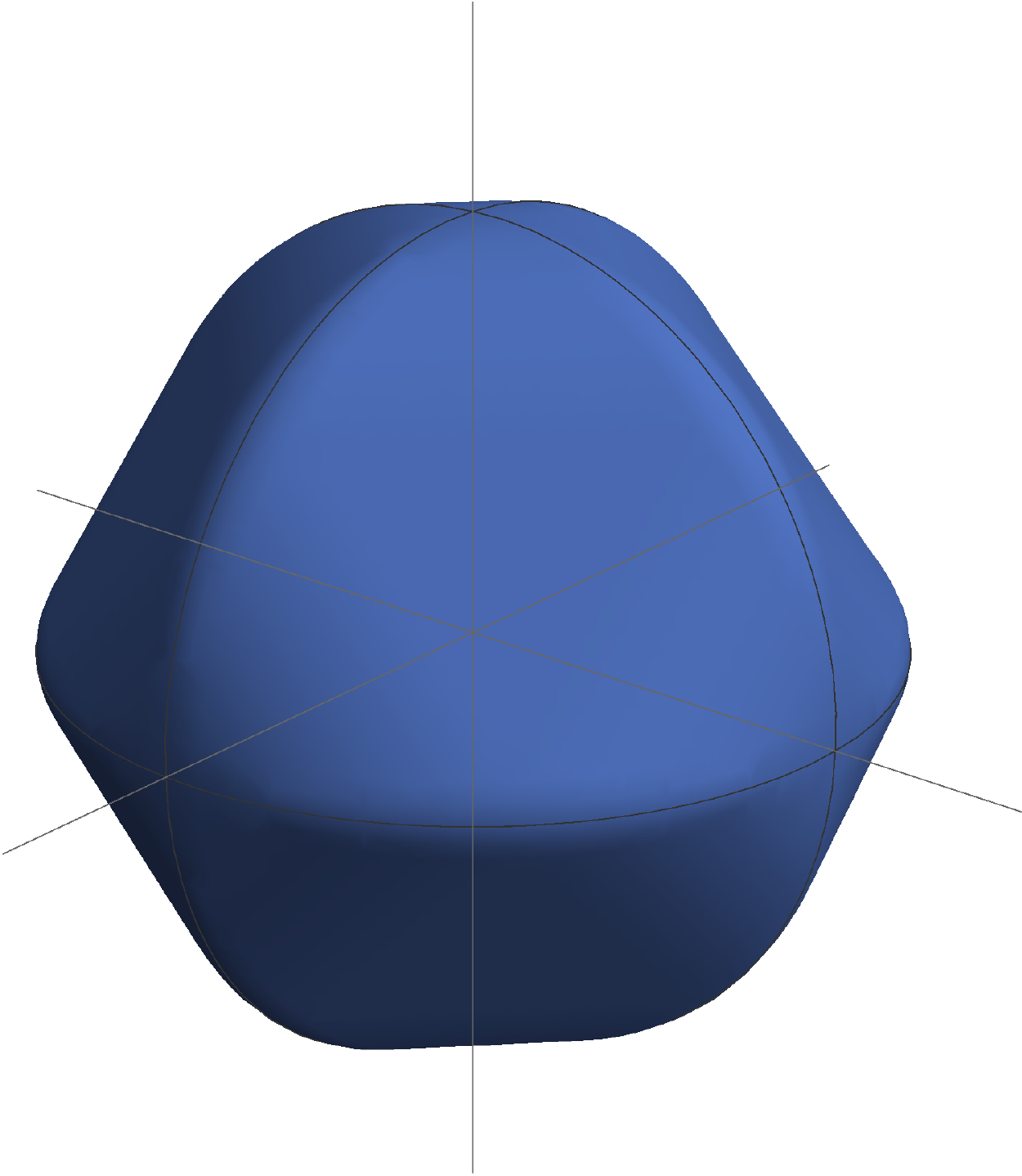}
\end{minipage}
\begin{minipage}{0.30\linewidth}
\centering
  \includegraphics[width=0.8\linewidth]{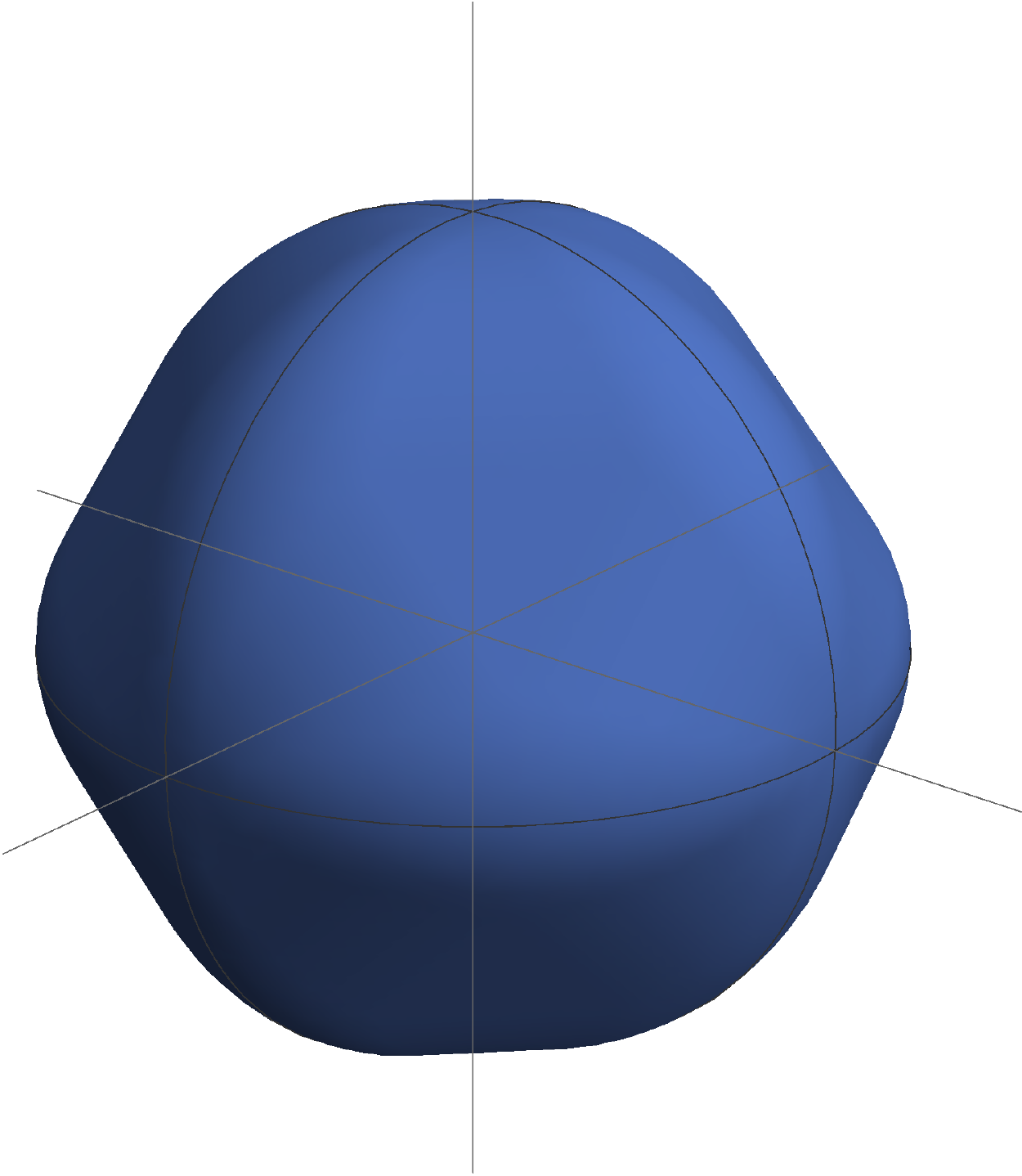}
\end{minipage}
\caption{Primal box-norm component ellipsoids, $(k,a)\in \{(1,0.15), (2,0.15),(2,0.40)\}$.\label{fig:3dbox-k1-primal-unit-balls}}
\centering
\begin{minipage}{0.30\linewidth}
\centering
  \includegraphics[width=0.8\linewidth]{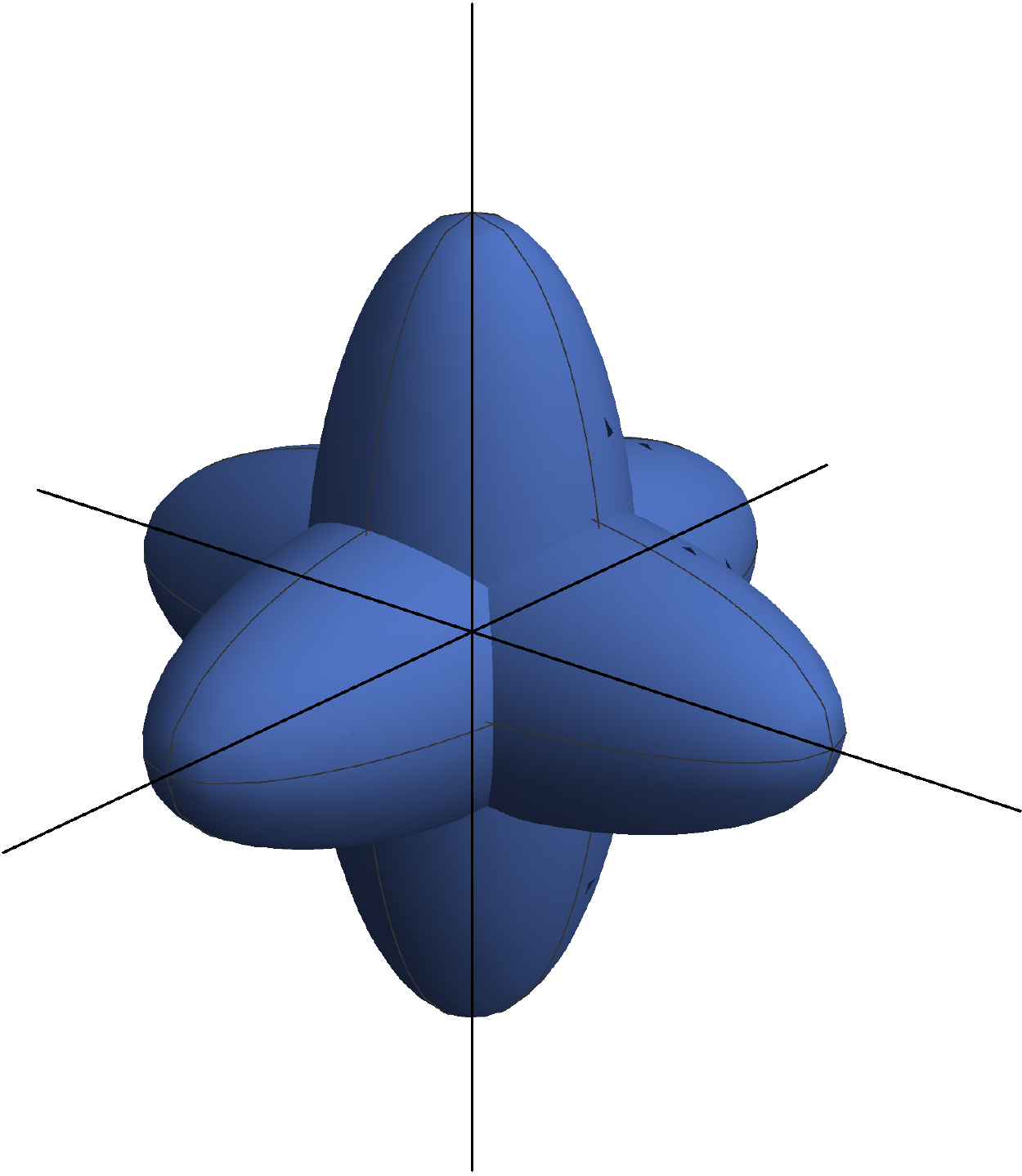}
\end{minipage}
\begin{minipage}{0.30\linewidth}
\centering
  \includegraphics[width=0.8\linewidth]{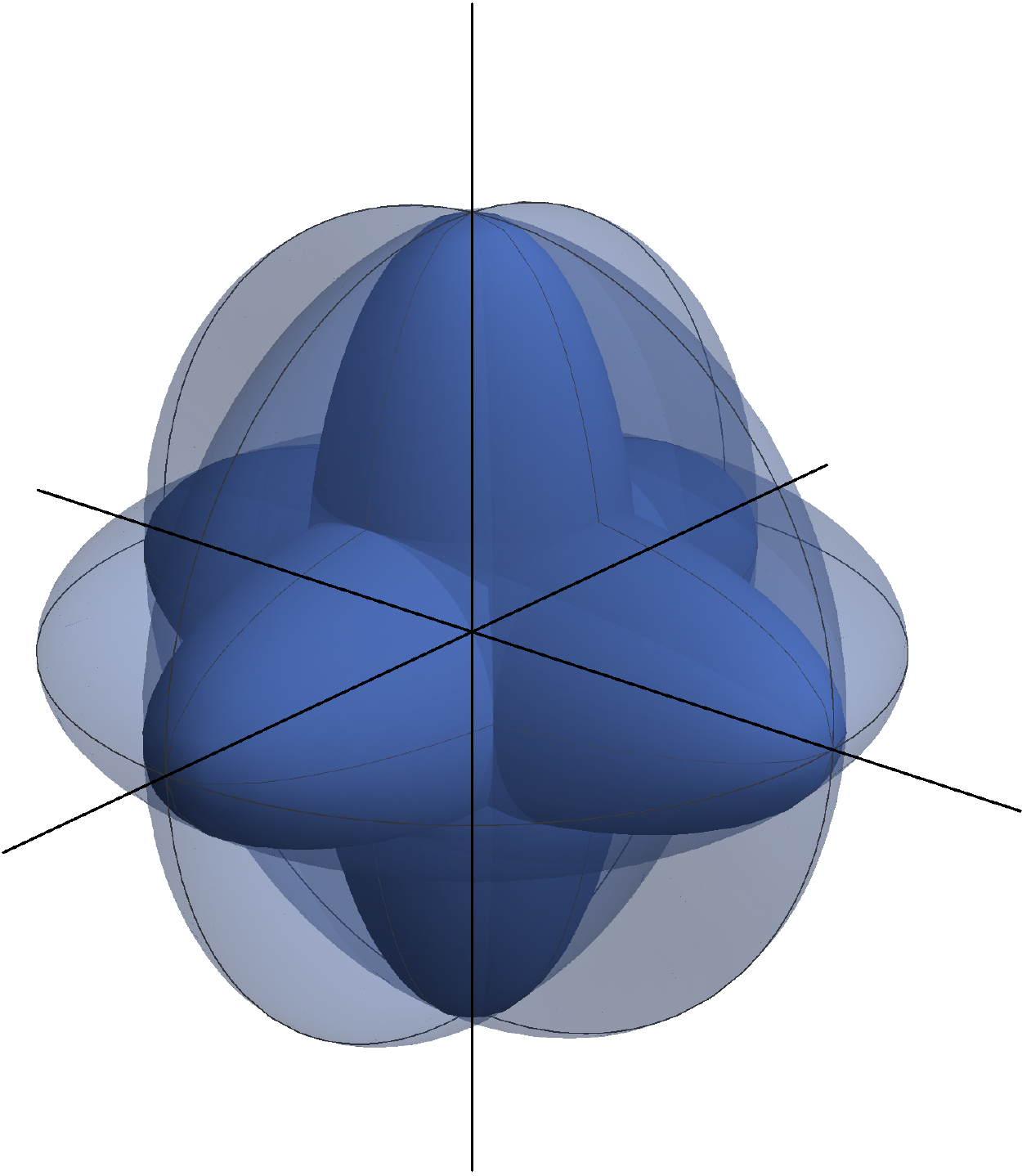}
\end{minipage}
\begin{minipage}{0.30\linewidth}
\centering
  \includegraphics[width=0.8\linewidth]{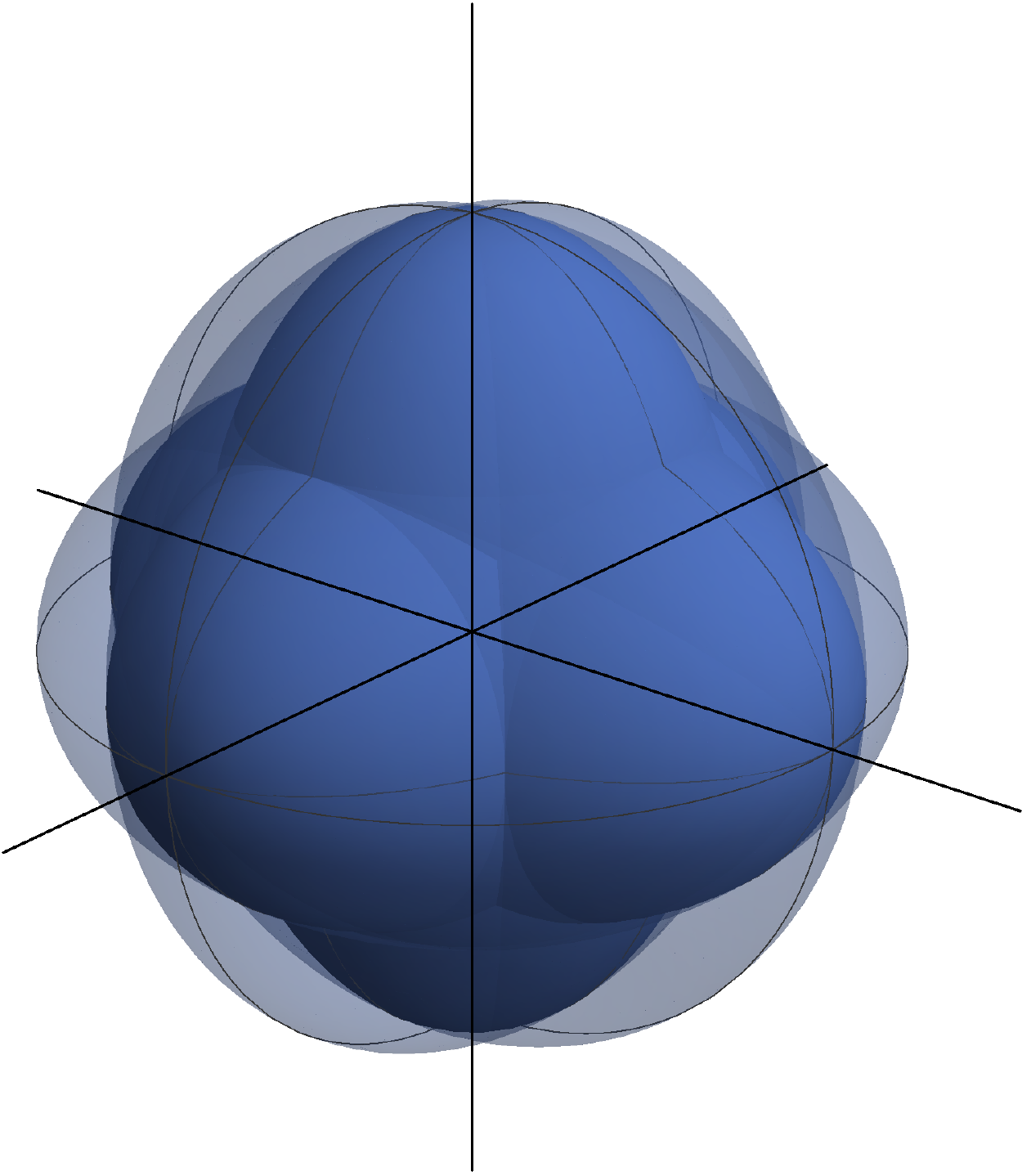}
\end{minipage}
\caption{Dual box-norm unit balls and ellipsoids, $(k,a)\in\{(1,0.15), (2,0.15),(2,0.40)\}$.  For $k=2$, only 3 tightest ellipsoids are shown. \label{fig:3dbox-k2-dual-unit-balls}}
\centering
\begin{minipage}{0.30\linewidth}
\centering
  \includegraphics[width=0.8\linewidth]{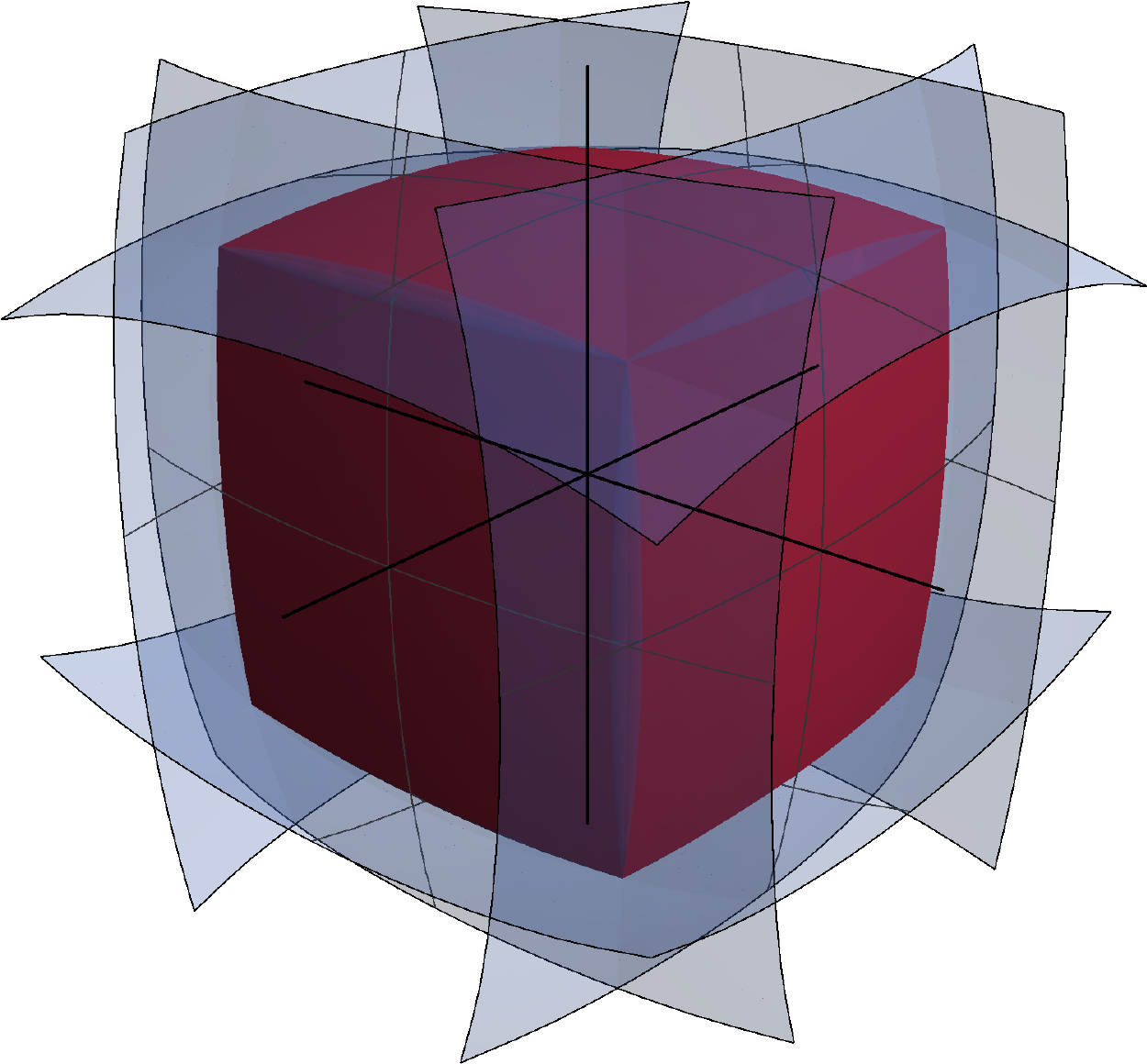}
\end{minipage}
\begin{minipage}{0.30\linewidth}
\centering
  \includegraphics[width=0.8\linewidth]{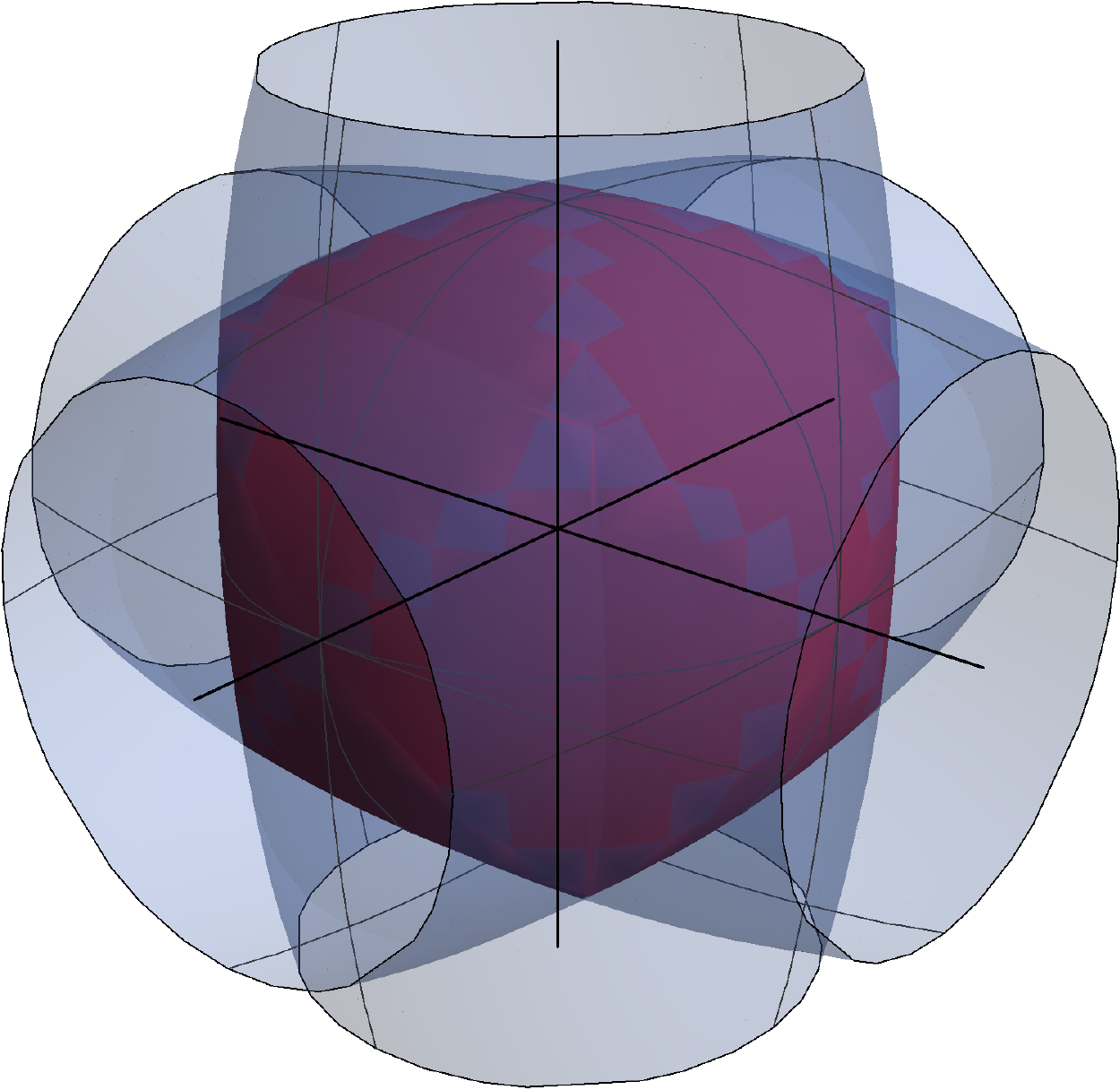}
\end{minipage}
\begin{minipage}{0.30\linewidth}
\centering
  \includegraphics[width=0.8\linewidth]{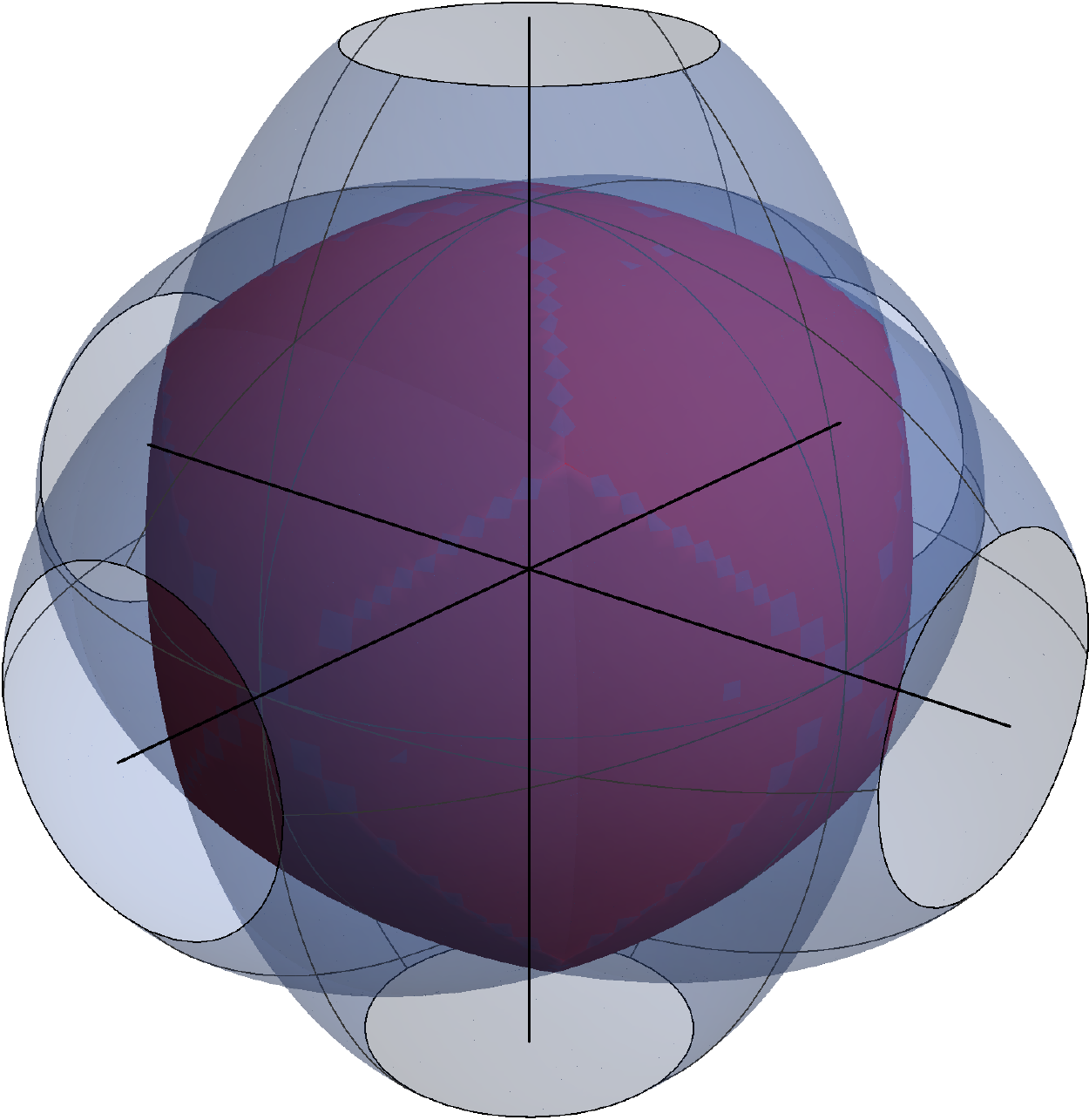}
\end{minipage}
\end{figure}


\section{Computation of the Norm and the Proximity Operator}
\label{sec:computation-of-norm-and-prox}
In this section, we compute the norm and the proximity operator of the squared box-norm by explicitly solving the optimization problem \eqref{eqn:theta-primal}. We also specialize our results to the $k$-support norm and comment on the improvement with respect the method by \cite{Argyriou2012}. Recall that, for every vector $w \in \R^d$, $|w|^{\downarrow}$ denotes the vector obtained from $w$ by reordering its components so that they are non-increasing in absolute value. 
\begin{theorem}\label{thm:computation-of-theta-norm}
For every $w \in \mathbb{R}^d$ it holds that 
\begin{align}
\Vert w \Vert_{\rm box}^2
&= \frac{1}{b}  \Vert w_Q \Vert_2^2  + 
\frac{1}{p} \Vert w_I \Vert_1^2 +
\frac{1}{a} \Vert w_L \Vert_2^2 \label{eqn:solution-of-abc-norm-2},
\end{align}
where 
$w_Q = (\vert w\vert^{\downarrow}_{1}, \ldots, \vert w\vert^{\downarrow}_{q})$, $w_I = (\vert w\vert^{\downarrow}_{q+1}, \ldots, \vert w\vert^{\downarrow}_{d-\ell})$, $w_L = (\vert w\vert^{\downarrow}_{d-\ell+1}, \ldots, \vert w\vert^{\downarrow}_{d})$, $q$ and $\ell$ are the unique integers in $\{0,\ldots, d\}$ that satisfy $q+\ell\leq d$,
\begin{align}
\frac{\vert w_q \vert}{b} \geq \frac{1}{p}\sum_{i=q+1}^{d-\ell}\vert w_i \vert > \frac{\vert w_{q+1} \vert}{b},
\quad ~~~
\frac{\vert w_{d-\ell} \vert}{a} \geq \frac{1}{p}\sum_{i=q+1}^{d-\ell}\vert w_i \vert > \frac{\vert w_{d-\ell+1} \vert}{a}, \label{eqn:optimal-ell-and-q}
\end{align}
$p=c-qb-\ell a$ and we have defined $\vert w_{0} \vert=\infty$ and $\vert w_{d+1} \vert = 0$. Furthermore, the minimizer $\theta$ has the form
\begin{align*}
\theta_i = 
\begin{cases}
b, \quad & \text{if } i \in \{1,\dots,q\}, \\
p \frac{\vert w_i \vert}{\sum_{j=q+1}^{d-\ell}\vert w_j \vert}, \quad &
\text{if } i \in \{q+1,\dots,d-\ell\},\\
a, \quad & \text{otherwise}.
\end{cases}
\end{align*}
\end{theorem}

\begin{proof}
We solve the constrained optimization problem
\begin{align}
\inf \bigg\{ \sum_{i=1}^d \frac{w_i^2}{\theta_i} : a \leq \theta_i \leq b, \sum_{i=1}^d \theta_i \leq c \bigg\}.  
\label{eqn:norm-objective-supp}
\end{align}
To simplify the notation we assume without loss of generality that $w_i$ are positive and ordered nonincreasing, and note that
the optimal $\theta_i$ are ordered non increasing.
To see this, 
let $\theta^* = \argmin_{\theta \in \Theta} \sum_{i=1}^d \frac{w_i^2}{\theta_i}$. 
Now suppose that $\theta^*_i < \theta_j^*$ for some $i < j$ and define $\hat{\theta}$ to be identical to $\theta^*$, except with the $i$ and $j$ elements exchanged.  
The difference in objective values is
\begin{align*}
\sum_{i=1}^d \frac{w_i^2}{{\hat \theta}_i} - \sum_{i=1}^d \frac{w_i^2}{{\theta_i^*}} = 
(w_i^2 - w_j^2)\left(\frac{1}{\theta_j^*} - \frac{1}{\theta_i^*} \right),
\end{align*}
which is negative so $\theta^*$ cannot be a minimizer. 
 
We further assume without loss of generality that $w_i \ne 0$ for all $i$, and $c \leq db$ (see Remark \ref{rem:wi-not-zero} below). 
The objective is continuous and we take the infimum over a closed bounded set, so a solution exists and it is unique by strict convexity.  
Furthermore, since $c \leq db$, the sum constraint will be tight at the optimum. 
Consider the Lagrangian function 
\begin{align}
L(\theta,\alpha) = \sum_{i=1}^d \frac{w_i^2}{\theta_i}  + \frac{1}{\alpha^2} \left( \sum_{i=1}^d \theta_i-c\right),
\label{eq:La}
\end{align}
where $1/\alpha^2$ is a strictly positive multiplier, and $\alpha$ is to be chosen to make the sum constraint tight, call this value $\alpha^*$.
Let $\theta^*$ be the minimizer of $L(\theta, \alpha^*) $ over $\theta$ subject to $a\leq \theta_i \leq b$. 

We claim that $\theta^*$ solves equation \eqref{eqn:norm-objective-supp}.  
Indeed, for any $\theta \in [a,b]^d$, $L(\theta^*,\alpha^*) \leq L(\theta, \alpha^*)$, which implies that
\begin{align*}
\sum_{i=1}^d \frac{w_i^2}{\theta^*_i}   \leq
\sum_{i=1}^d \frac{w_i^2}{\theta_i}  + \frac{1}{(\alpha^*)^2} \left( \sum_{i=1}^d \theta_i-c\right).
\end{align*}
If in addition we impose the constraint $\sum_{i=1}^d \theta_i \leq c$, the second term on the right hand side is at most zero, so we have for all such $\theta$ that
\begin{align*}
\sum_{i=1}^d \frac{w_i^2}{\theta^*_i}   \leq \sum_{i=1}^d \frac{w_i^2}{\theta_i},
\end{align*}
whence it follows that $\theta^*$ is the minimizer of \eqref{eqn:norm-objective-supp}.

We can therefore solve the original problem by minimizing the Lagrangian \eqref{eq:La} over the box constraint.
Due to the coupling effect of the multiplier, the problem is separable, and we can solve the simplified problem componentwise \citep[see][Theorem 3.1]{Micchelli2013}. 
For completeness we repeat the argument here. For every $w_i \in \R$ and $\alpha > 0$, the unique solution to the problem $\min\{ \frac{w_i^2}{\theta} +  \frac{\theta}{\alpha^2} : a\leq \theta \leq b\}$ is given by 
\begin{align}
\theta = 
\begin{cases}
b, \quad &\text{if }  \alpha \vert w_i \vert> b,\\
\alpha \vert w \vert, &\text{if } b \geq \alpha \vert w_i \vert\geq a,\\
a, \quad &\text{if } a> \alpha \vert w_i \vert.
\end{cases}\label{eqn:form-of-optimal-theta-abc}
\end{align}
Indeed, for fixed $w_i$, the objective function is strictly convex on $\mathbb{R}^d_{++}$ and has a unique minimum on $(0,\infty)$ (see Figure 1.b in \citet{Micchelli2013} for an illustration).  
The derivative of the objective function is zero for $\theta = \theta^* := \alpha \vert w_i \vert$, strictly positive below $\theta^*$ and strictly increasing above $\theta^*$.  
Considering these three cases the result follows and $\theta$ is determined by \eqref{eqn:form-of-optimal-theta-abc} where $\alpha$ satisfies $\sum_{i=1}^d \theta_i(\alpha) = c$.

The minimizer then has the form 
\begin{align*}
\theta = (\underbrace{b, \ldots, b}_q, \theta_{q+1}, \ldots, \theta_{d-\ell}, \underbrace{a, \ldots, a}_{\ell}),
\end{align*}
where $q,\ell \in \{0, \ldots, d\}$ are determined by the value of $\alpha$ which satisfies
\begin{align*}
S(\alpha) = \sum_{i=1}^d \theta_i(\alpha) =q b + \sum_{i=q+1}^{d-\ell} \alpha \vert w_i \vert +  \ell a = c,
\end{align*}
i.e. $\alpha = p / \left(\sum_{i=q+1}^{d-\ell}\vert w_i \vert\right)$, where $p=c-qb -\ell a$.

The value of the norm follows by substituting $\theta$ into the objective and we get
\begin{align*}
\Vert w \Vert_{{\rm box}}^2 
&= \sum_{i=1}^{q} \frac{\vert w_i \vert^2}{b} + 
\frac{1}{p} \Big(\sum_{i=q+1}^{d-\ell} \vert w_i\vert\Big)^2 +  \sum_{i=d-\ell+1}^{d} \frac{\vert w_i\vert^2}{a} = \frac{1}{b} \Vert w_Q \Vert_2^2  + 
\frac{1}{p} \Vert w_I\Vert_1^2 +
\frac{1}{a} \Vert w_L \Vert_2^2, 
\end{align*}
as required. 
We can further characterize $q$ and $\ell$ by considering the form of $\theta$.  
By construction we have $\theta_q \geq b >\theta_{q+1}$ and $\theta_{d-\ell} > a \geq \theta_{d-\ell+1}$, or equivalently
\begin{align*}
\frac{\vert w_q \vert}{b} &\geq \frac{1}{p}\sum_{i=q+1}^{d-\ell}\vert w_i \vert > \frac{\vert w_{q+1} \vert}{b} ~~~\textrm{and}~~~ \frac{\vert w_{d-\ell} \vert}{a} \geq \frac{1}{p}\sum_{i=q+1}^{d-\ell}\vert w_i \vert > \frac{\vert w_{d-\ell+1} \vert}{a}.
\end{align*}
The proof is completed.
\end{proof}

\begin{remark}\label{rem:wi-not-zero}
The case where some $w_i$ are zero follows from the case that we have considered in the theorem.  
If $w_i=0$ for $n < i \leq d$, then clearly we must have $\theta_i=a$ for all such $i$.  
We then consider the $n$-dimensional problem of finding $(\theta_1, \ldots, \theta_{n})$ that minimizes $\sum_{i=1}^{n} \frac{w_i^2}{\theta_i}$, subject to $a \leq \theta_i \leq b$, and $\sum_{i=1}^{n} \theta_i \leq c'$, where $c'=c-(d-n)a$.  
As $c \geq da$ by assumption, we also have $c' \geq na$, so a solution exists to the $n$-dimensional problem.  
If $c' \geq bn$, then a solution is trivially given by $\theta_i=b$ for all $i=1,\dots,n$.  
In general, $c' < bn$, and we proceed as per the proof of the theorem. 
Finally, a vector that solves the original $d$-dimensional problem will be given by $(\theta_1, \ldots, \theta_{n}, a, \ldots, a)$. 
\end{remark}

Theorem \ref{thm:computation-of-theta-norm} suggests two methods for computing the box-norm.  
First, we can find $\alpha$ such that $S(\alpha)=c$; this value uniquely determines $\theta$ in \eqref{eqn:form-of-optimal-theta-abc}, and the norm follows by substitution into 
the objective in \eqref{eq:La}.
Alternatively, we identify $q$ and $\ell$ that jointly satisfy \eqref{eqn:optimal-ell-and-q} and we compute the norm using  \eqref{eqn:solution-of-abc-norm-2}.  
Taking advantage of the structure of $\theta$ in the former method leads to a computation time that is $\mathcal{O}( d \log d)$. 
 
\begin{theorem}\label{thm:theta-norm-d-log-d}
The computation of the box-norm can be completed in $\mathcal{O}( d \log d)$ time. 
\end{theorem}

\begin{proof}
Following Theorem \ref{thm:computation-of-theta-norm}, we need to determine $\alpha^*$ to satisfy the coupling constraint $S(\alpha^*)  = c$.
{Each component $\theta_i$ is a piecewise linear function in the form of a step function with a constant positive slope between the values $a/\vert w_i\vert$ and $b/\vert w_i\vert$.  }
Let $\left\{ \alpha^{i} \right\}_{i=1}^{2d}$ be the set of the $2d$ critical points, where the $\alpha^{i}$ are ordered nondecreasing.  
The function $S(\alpha)$ is a nondecreasing piecewise linear function with at most $2d$ critical points. We can find $\alpha^*$ by first sorting the points $\{ \alpha^i \}$, finding $\alpha^i$ and $\alpha^{i+1}$ 
such that 
\begin{align*}
S(\alpha^i) \leq c \leq S(\alpha^{i+1})
\end{align*}
by binary search, and then interpolating $\alpha^*$ between the two points. 
Sorting takes $\mathcal{O}(d \, \log d)$.  
Computing $S(\alpha^i)$ at each step of the binary search is $\mathcal{O}(d )$, so $\mathcal{O}(d \, \log d)$ overall.  
Given $\alpha^i$ and $\alpha^{i+1}$, interpolating $\alpha^*$ is $\mathcal{O}(1)$, so the overall algorithm is $\mathcal{O}(d \, \log d)$ as claimed.
\end{proof}

The $k$-support norm is a special case of the box-norm, and as a direct corollary of Theorem \ref{thm:computation-of-theta-norm} and Theorem \ref{thm:theta-norm-d-log-d}, we recover \citep[][Proposition 2.1]{Argyriou2012}.

\begin{corollary}\label{cor:ksup-norm}
For $w \in \mathbb{R}^d$, and $k\leq d$, 
\begin{align*}
\Vert w \Vert_{(k)} &=  \Bigg( \sum_{j=1}^q ( \vert w\vert^{\downarrow}_j )^2  + 
\frac{1}{k-q} \Big(\sum_{j=q+1}^d \vert w^{\downarrow}_j \vert\Big)^2 \Bigg)^{\frac{1}{2}},
\end{align*}
where $q$ is the unique integer in $\{0, k-1\}$ satisfying
\begin{align}
\vert w_{q} \vert \geq \frac{1}{k-q} \sum_{j=q+1}^d \vert w_{j} \vert > \vert w_{q+1} \vert, \label{eqn:k-sup-optimal-q}
\end{align}
and we have defined $w_0=\infty$.  Furthermore, the norm can be computed in $\mathcal{O}(d \log d)$ time. 
\end{corollary}


\subsection{Proximity Operator}\label{sec:computation-of-prox}
Proximal gradient methods can be used to solve optimization problems of the form 
\begin{align*}
\min_{w} f(w) + \lambda g(w), ~~~ w \in \R^d,
\end{align*}
where $f$ is a convex loss function with Lipschitz continuous gradient, $\lambda >0$ is a regularization parameter, and $g$ is a convex function for which the proximity operator can be computed efficiently, see \citet{Nesterov2007,Combettes2010, Beck2009} and references therein. 
The proximity operator of $g$ with parameter $\rho>0$ is defined as 
\begin{align*}
\prox_{\rho g} (w) = \argmin \left\{
\frac{1}{2} \Vert x-w\Vert^2 + \rho g(x) : x \in {\mathbb R}^d \right\}.
\end{align*}
We now use the infimum formulation of the box-norm to derive the proximity operator of the squared norm.

\begin{theorem}\label{thm:prox-of-general-theta-norm}
The proximity operator of the square of the box-norm at point $w \in \mathbb{R}^d$ with parameter $\frac{\lambda}{2}$ is given by
$\prox_{\frac{\lambda}{2}\Vert \cdot \Vert_{{\rm box}}^2}(w) = (\frac{\theta_1 w_1}{\theta_1+\lambda},\dots, \frac{\theta_d w_d}{\theta_d+\lambda})$, where
\begin{align}
\theta_i &= 
\begin{cases}
    		b, & \text{if } \, \alpha \vert w_i \vert -\lambda  > b, \\
    		\alpha \vert w_i \vert - \lambda, & \text{if } \,  
    		b \geq  \alpha {\vert w_i \vert}- \lambda  \geq a ,\\
    		a, & \text{if } \, a > \alpha\vert w_i \vert  - \lambda,
  	\end{cases} \notag
\end{align}
and $\alpha$ is chosen such that $S(\alpha) := \sum_{i=1}^d \theta_i(\alpha) = c$. 
Furthermore, the computation of the proximity operator can be completed in $\mathcal{O}( d \log d)$ time. 
\end{theorem}


\begin{proof}
Using the infimum formulation of the norm, we solve
\begin{align*}
\min_{x\in \mathbb{R}^d} \inf_{\theta \in \Theta} \, 
\left\{\frac{1}{2}\sum_{i=1}^d (x_i-w_i)^2 + 
\frac{\lambda}{2}\sum_{i=1}^d \frac{x_i^2}{\theta_i} \right\}.
\end{align*}
We can exchange the order of the optimization and solve for $x$ first.  
The problem is separable and a direct computation yields that $x_i = \frac{\theta_i w_i}{\theta_i + \lambda}$. 
Discarding a multiplicative factor of $\lambda/2$, and noting that the infimum is attained, the problem in $\theta$ becomes
\begin{align*}
\min_{\theta} \bigg\{ 
\sum_{i=1}^d \frac{ w_i^2}{\theta_i+\lambda} : {a \leq \theta_i \leq b, \sum_{i=1}^d \theta_i \leq c} \,\bigg\}.
\end{align*}
Note that this is the same as computing a box-norm in accordance with Proposition \ref{cor:box_as_k_sup_and_l2}.
Specifically, this is exactly like problem \eqref{eqn:norm-objective-supp} after the change of variable $\theta'_i = \theta_i + \lambda$. The remaining part of the proof then follows in a similar manner to the proof of Theorem \ref{thm:computation-of-theta-norm}.
\end{proof}

Algorithm \ref{alg:prox_01} illustrates the computation of the proximity operator for the squared box-norm in $\mathcal{O}(d \log d)$ time.  
This includes the $k$-support as a special case, where we let $a$ tend to zero, and set $b=1$ and $c=k$, which improves upon the complexity of the $\mathcal{O}(d(k + \log d))$ computation provided in \citet{Argyriou2012}, and   
we illustrate the improvement empirically in Table \ref{table:prox_comparison}.
We summarize this in the following corollary.
\begin{corollary}\label{cor:prox-of-ksup}
The proximity operator of the square of the $k$-support norm at point $w$ with parameter $\frac{\lambda}{2}$ is given by
$\prox_{\frac{\lambda}{2}\Vert \cdot \Vert_{\Theta}^2}(w) = x$, 
where $x_i =  \frac{\theta_i w_i}{\theta_i+\lambda}$, and
\begin{align*}
\theta_i &= 
\begin{cases}
    		1, & \text{if } \, \alpha \vert w_i \vert  > \lambda+1   , \\
    		\alpha \vert w_i \vert - \lambda, & \text{if } \,  
    		\lambda +1 \geq \alpha \vert w_i \vert  \geq \lambda\\
    		0, & \text{if } \, \lambda >  \alpha \vert w_i \vert  ,
  	\end{cases}
\end{align*}
where $\alpha$ is chosen such that $S(\alpha) = k$.
Furthermore, the proximity operator can be computed in $\mathcal{O}(d \log d)$ time. 
\end{corollary}

\begin{algorithm}[t]
\caption{Computation of $x= \prox_{\frac{\lambda}{2}\|\cdot \|_{\rm box}^2}\left(w\right)$. \label{alg:prox_01}}
\begin{algorithmic} 
\REQUIRE parameters  $a$, $b$, $c$, $\lambda$.
\STATE \textbf{1.} Sort points $\left\{ \alpha^i \right\}_{i=1}^{2d} = \left\{ \frac{a+\lambda}{ \vert w_j \vert}, \frac{b+\lambda}{ \vert w_j \vert} \right\}_{j=1}^d$  such that $\alpha^i \leq \alpha^{i+1}$;
\STATE \textbf{2.} Identify points $\alpha^i$ and $\alpha^{i+1}$ such that $S(\alpha^i) \leq c$ and $S(\alpha^{i+1})\geq c$ by binary search;
\STATE \textbf{3.} Find $\alpha^*$ between $\alpha^i$ and $\alpha^{i+1}$ such that $S(\alpha^*)=c$ by linear interpolation;
\STATE \textbf{4.} Compute $\theta_i(\alpha^*)$ for $i=1\ldots, d$;
\STATE \textbf{5.} Return $x_i =\frac{\theta_i w_i}{\theta_i+\lambda}$ for $i=1\ldots, d$.
\end{algorithmic}
\end{algorithm}


\section{Spectral Norms}
\label{sec:matrix-norms}
We now turn our focus to the matrix norms.  
For this purpose, we recall that a norm $\Vert\cdot \Vert$ on $\mathbb{R}^{d \times T}$ is called orthogonally invariant if $\Vert W \Vert= \Vert U W V \Vert$, for any orthogonal matrices $U \in \mathbb{R}^{d \times d}$ and $V \in \mathbb{R}^{T \times T}$. 
A classical result by \citet{VonNeumann1937} establishes that a norm is orthogonally invariant if and only if it is of the form $\Vert W \Vert = g(\sigma(W))$, where $\sigma(W)$ is the vector formed by the singular values of $W$ in nonincreasing order, and $g$ is a symmetric gauge function, that is a norm which is invariant under permutations and sign changes of the vector components.  

\begin{lemma}\label{lem:ksup-is-OI}
If~$\Theta$ is a convex bounded subset of the strictly positive orthant in $\mathbb{R}^d$ which is invariant under permutations, then $\|\cdot\|_\Theta$ is a symmetric gauge function.  
\end{lemma}

\begin{proof}
Let $g(w) = \Vert w \Vert_{\Theta}$.  
We need to show that $g$ is a norm which is invariant under permutations and sign changes. 
By Proposition \ref{prop:theta-is-norm}, $g$ is a norm, so it remains to show that $g(w) = g(Pw)$ for every permutation matrix $P$, and $g(Jw) = g(w)$ for every diagonal matrix $J$ with entries $\pm 1$.  
The former property follows since the set $\Theta$ is permutation invariant.
The latter property is true because the objective function in \eqref{eqn:theta-primal} involves the squares of the components of $w$. 
\end{proof}

In particular, this readily applies to both the $k$-support norm and the box-norm. We can therefore extend both norms to orthogonally invariant norms, which we term the spectral $k$-support norm and the spectral box-norm respectively, and which we write (with some abuse of notation) as $\Vert W \Vert_{(k)} = \Vert \sigma(W) \Vert_{(k)}$ and $\Vert W \Vert_{\rm box} = \Vert \sigma(W) \Vert_{\rm box}$. 
We note that since the $k$-support norm subsumes the $\ell_1$ and $\ell_2$-norms for $k=1$ and $k=d$ respectively, the corresponding spectral $k$-support norms are equal to the trace and Frobenius norms respectively.

A number of properties of the vector norms translate in the natural manner to the matrix norms. 
We first characterize the unit ball of the spectral $k$-support norm.

\begin{proposition}\label{prop:spectral-unit-ball}
The unit ball of the spectral $k$-support norm is the convex hull of the set of matrices of rank at most k and Frobenius norm no greater than one.
\end{proposition}

\begin{proof}
For any $W \in \mathbb{R}^{d\times T}$, define the following sets
\begin{align*}
T_k &= \{W \in \mathbb{R}^{d\times T} : \textrm{rank}(W) \leq k, ~ \Vert W \Vert_F \leq 1\} ,\quad A_k = \textrm{co}(T_k),
\end{align*}
and consider the following functional
\begin{align}
\lambda(W) 
&= \inf \{ \lambda>0: W \in \lambda A_k \}, ~~~ W\in \mathbb{R}^{d \times T}   \label{eqn:k-sup-matrix-def-2}.
\end{align}
We will apply Lemma \ref{lem:minkowski-bounded} in the appendix to the set $A_k$.  
To do this, we need to show that the set $A_k$ is bounded, convex, symmetric and absorbing.
The first three are clearly satisfied. 
To see that it is absorbing, let $W \in \mathbb{R}^{d \times T}$ have singular value decomposition $U \Sigma V\trans$, and let $r=\min(d,T)$. 
If $W$ is zero then clearly $W\in A_k$, so assume it is non zero. 

{For $i \in \N_r$ let $S_i\in\R^{d\times T}$ have entry $(i,i)$ equal to $1$, and all remaining entries zero.
We then have
\begin{align*}
W&=  U \Sigma V\trans 
= U \left(\sum_{i=1}^r \sigma_i S_i  \right) V\trans 
= \left(\sum_{i=1}^d \sigma_i \right)   \sum_{i=1}^r \frac{\sigma_i}{\sum_{j=1}^r \sigma_j}(U S_i V \trans) 
=: \lambda \sum_{i=1}^r \beta_i Z_i.
\end{align*}
Now for each $i$, $\Vert Z_i \Vert_F = \Vert S_i \Vert_F = 1$, and $\textrm{rank}(Z_i) = \textrm{rank}(S_i) = 1$, so $Z_i \in T_k$ for any $k \geq 1$. 
Furthermore $\beta_i \in [0,1]$ and $\sum_{i=1}^r \beta_i=1$, that is $(\beta_1, \ldots, \beta_r) \in \Delta^{r-1}$, so $\frac{1}{\lambda}W$ is a convex combination of $Z_i$, in other words $W \in \lambda A_k$, and we have shown that $A_k$ is absorbing.
It follows that $A_k$ satisfies the hypotheses of Lemma \ref{lem:minkowski-bounded}, where we let $C=A_k$, hence $\lambda$ defines a norm on $\R^{d \times T}$ with unit ball equal to $A_k$.  }

Since the constraints in $T_k$ involve spectral functions, the sets $T_k$ and $A_k$ are invariant to left and right multiplication by orthogonal matrices. 
It follows that $\lambda$ is a spectral function, that is $\lambda(W)$ is defined in terms of the singular values of $W$.
By von Neumann's Theorem \citep{VonNeumann1937} the norm it defines is orthogonally invariant and we have
\begin{align}
\lambda(W) &= \inf \{ \lambda>0: W \in \lambda A_k \} 
= \inf \{ \lambda >0: \sigma(W) \in \lambda C_k \}  
 = \Vert \sigma(W) \Vert_{(k)} \notag
\end{align}
where we have used Corollary \ref{prop:vector-ksup-Ck-def}, which states that $C_k$ is the unit ball of the $k$-support norm. 
It follows that the norm defined by \eqref{eqn:k-sup-matrix-def-2} is the spectral $k$-support norm with unit ball given by $A_k$.
\end{proof}

Referring to the unit ball characterization of the $k$-support norm, we note that the restriction on the cardinality of 
the vectors which define the extreme points of the unit ball naturally extends to a restriction on the rank operator in the matrix setting.
Furthermore, as noted by \citet{Argyriou2012}, regularization using the $k$-support norm encourages vectors to be sparse, but less so than the $\ell_1$-norm. 
In matrix regularization problems, Proposition \ref{prop:spectral-unit-ball} suggests that the spectral $k$-support norm for $k>1$ encourages matrices to have low rank, but less so than the trace norm. 
This is intuitive as the extreme points of the unit ball have rank at most $k$.

As in the case of the vector norm (Proposition \ref{cor:box_as_k_sup_and_l2}), the spectral box-norm 
(or cluster norm -- see below) can be written as a perturbation of the spectral $k$-support norm with a quadratic term.

\begin{proposition}\label{prop:matrx_box_as_k_sup_and_l2}
Let $\Vert \cdot \Vert_{{\rm box}}$ be a matrix box-norm with parameters $a,b,c$ and let $k = \frac{c-da}{b-a}$. Then
\begin{align*}
\Vert W \Vert_{{\rm box}}^2 &= \min_{Z \in \R^{d\times T}}  \left\{ \frac{1}{a} \Vert W - Z \Vert_F^2 + \frac{1}{b-a} \Vert Z \Vert_{(k)}^2 \right\}.
\end{align*}
\end{proposition}

\begin{proof}
By von Neumann's trace inequality (Theorem \ref{thm:vonneumann} in the appendix) we have
\begin{align*}
\frac{1}{a}  \| W-Z \Vert_F^2 + \frac{1}{b-a}  \Vert Z \|_{(k)}^2 
&= \frac{1}{a} \left( \Vert W \Vert_F^2 + \Vert Z \Vert_F^2 - 2 \langle W,Z \rangle\right) +  \frac{1}{b-a} \Vert Z \Vert_{(k)}^2 \\  
&\geq \frac{1}{a}  \left( \Vert \sigma(W) \Vert_2^2 + \Vert \sigma(Z) \Vert_2^2 - 2 \langle \sigma(W) , \sigma(Z)  \rangle\right) +  \frac{1}{b-a} \Vert \sigma(Z) \Vert_{(k)}^2\\  
&= \frac{1}{a}  \Vert \sigma(W)-\sigma(Z) \Vert_2^2 +  \frac{1}{b-a} \Vert \sigma(Z) \Vert_{(k)}^2 .  \\
\end{align*}

Furthermore the inequality is tight if $W$ and $Z$ have the same ordered set of singular vectors. Hence 
\begin{align*}
\min_{Z \in \R^{d\times T}} \left\{ \frac{1}{a}  \| W-Z \Vert_F^2 + \frac{1}{b-a}  \Vert Z \|_{(k)}^2\right\} 
&= \min_{z \in \R^d} \left\{\frac{1}{a}  \Vert \sigma(W)-z \Vert_2^2 +  \frac{1}{b-a} \Vert z \Vert_{(k)}^2 \right\}
= \|\sigma(W)\|_{{\rm box}}^2,
\end{align*}
where the last equality follows by Proposition \ref{cor:box_as_k_sup_and_l2}.
\end{proof}

In other words, this result shows that the (scaled) squared spectral box-norm can be seen as the Moreau envelope  
of a squared spectral $k$-support norm. 


\subsection{Proximity Operator for Orthogonally Invariant Norms} 
\label{sec:centered-prox}
The computational considerations outlined in Section \ref{sec:computation-of-norm-and-prox} can be naturally extended to the matrix setting by using von Neumann's trace inequality stated in the appendix. Here we comment on the computation of the proximity operator, which is important for our numerical experiments in Section \ref{sec:numerics}. 
The proximity operator of an orthogonally invariant norm $\Vert \cdot \Vert = g( \sigma(\cdot))$ is given by
\begin{align*}
\prox_{\Vert \cdot \Vert} (W)= U \text{diag}(\prox_{g}(\sigma(W))) V\trans, ~~~ W \in \mathbb{R}^{m \times T},
\end{align*}
where $U$ and $V$ are the matrices formed by the left and right singular vectors of $W$ \citep[see e.g.][Prop. 3.1]{Argyriou2011}. Using this result we can employ proximal gradient methods to solve matrix regularization problems using the square of the spectral $k$-support and box-norms.


\section{Multitask Learning}
\label{sec:MTL}
In this section, we address multitask learning, a framework in which spectral regularizers have successfully been used to learn a set of regression or binary classification tasks. 
Within this setting each column of the matrix $W$ represents one of the task weight vectors.  
By leveraging the commonalities between the tasks, learning can often be improved compared to solving each task in isolation \citep[see e.g.][and references therein]{Evgeniou2005, Argyriou2006, Argyriou2008, Jacob2009-CLUSTER,Cavallanti}.  
A natural assumption that arises in applications is that the tasks are clustered. 
The cluster norm was introduced by \citet{Jacob2009-CLUSTER} as a means to favour this structure. 
We show that this norm is equivalent to the spectral box-norm and then address the issue of centering the norm.

\subsection{Clustering the Tasks}
A general approach to multitask learning is based on the regularization problem
\begin{align*}
\min_{W \in \mathbb{R}^{d \times T}} \calL(W) + \lambda \Omega(W)
\end{align*}
where $W = [w_1, \ldots, w_T]$ is the $d \times T$ matrix whose columns represent the task vectors, $\Omega$ is a regularizer which incorporates prior knowledge of sharing between tasks and $\calL$ is the compound empirical error. 
That is, $\calL(W) = \frac{1}{Tn} \sum_{t=1}^T \sum_{i =1}^{n} \ell(y_{i}^t, \lb w_t, x_i^t \rb)$
where $(x_1^t,y_1^t),\dots,(x_n^t,y_n^t) \in  \mathbb{R}^d \times \mathbb{R}$ are the training points for task $t$ ({for simplicity we assume that each task has the same number $n$ of training points) and $\ell$ is a convex loss function.
  
\citet{Jacob2009-CLUSTER} consider a composite penalty which encourages the tasks to be clustered into $Q<T$ groups. To introduce their setting we require some more notation.
Let $\mathcal{J}_q \subseteq \NN{T}$ be the set of tasks in cluster $q \in \NN{Q}$ and let $T_q = \vert \mathcal{J}_q \vert \geq 0$ be the number of tasks in cluster $q$, so that $\sum_{q=1}^Q T_q = T$.  
The clustering uniquely defines the $T \times T$ normalized connectivity matrix $M$ where $M_{st} =\frac{1}{T_q}$ if $s,t \in {\cal J}_q$ and $M_{st}=0$ otherwise. 
We let $\bar{w} = \frac{1}{T}\sum_{t=1}^{T} w_t$ be the mean weight vector, $\bar{w}_q = \frac{1}{T_q}\sum_{t\in \mathcal{J}_q} w_t$ be the mean weight vector of tasks in cluster $q$ and define the $T \times T$ orthogonal projection matrices $U = 1 1\trans / T$ and $\Pi = I-U$. 
Note that $W\Pi = [w_1 - {\bar w},\dots,w_T - {\bar w}]$. 
Finally, let $r=\min(d,T)$.

Using this notation, we introduce the three seminorms
\begin{align}
\nonumber
\Omega_{\rm m}(W) &= T \Vert \bar{w} \Vert^2 = \tr \left( WUW\trans \nonumber \right)\\
\Omega_{\rm b}(W) &= \sum_{q=1}^Q T_q \Vert \bar{w}_q - \bar{w} \Vert^2 = \tr  \left(W(M-U)W\trans\right) \nonumber \\
\Omega_{\rm w}(W) &= \sum_{q=1}^Q \sum_{t \in \mathcal{J}_q} \Vert w_t - \bar{w}_q \Vert^2 
=\tr  \left(W(I-M)W\trans \right),\nonumber
\end{align}
each of which captures a different aspect of the clustering: $\Omega_{\rm m}$ penalizes the total {\em mean} of the weight vectors, $\Omega_{\rm b}$ measures how close to each other the clusters are ({\em between} cluster variance), and $\Omega_{\rm w}$ measures the compactness of the clusters ({\em within} cluster variance). 
Scaling the three penalties by positive parameters $\epsilon_{\rm m}$, $\epsilon_{\rm b}$ and $\epsilon_{\rm w}$ respectively, 
we obtain the composite penalty $\epsilon_{\rm m} \Omega_{\rm m} + \epsilon_{\rm b} \Omega_{\rm b} + \epsilon_{\rm w} \Omega_{\rm w}$. 
The first term $\Omega_{\rm m}$ does not depend on the connectivity matrix $M$, and it can be included in the error term. The remaining two terms depend on $M$, which in general may not be known {\em a-priori}. \citet{Jacob2009-CLUSTER} propose to learn the clustering by minimizing with respect to matrix $M$, under the assumption that $\epsilon_{\rm w} \geq \epsilon_{\rm b}$; {this assumption is reasonable as we care more about enforcing a small variance of parameters within the clusters than between them.} 
Using the elementary properties that $M-U = M \Pi= \Pi M \Pi$ and $I-M = (I-M) \Pi = \Pi (I-M) \Pi$ and letting $\widetilde{M} = M \Pi$, we rewrite 
\begin{align}
\epsilon_{\rm b} \Omega_{\rm b}(W) + \epsilon_{\rm w} \Omega_{\rm w}(W) 
= \tr \left( {W} \Pi \big( \epsilon_{\rm b} \widetilde{M} + \epsilon_{\rm w}(I-\widetilde{M}) \big) \Pi {W}\trans \right)
= \tr \left( {W} \Pi {\Sigma}^{-1} \Pi {W}\trans\right) 
\label{eq:gggg}
\end{align}
where we have defined $\Sigma^{-1} = \epsilon_{\rm b} \widetilde{M} + \epsilon_{\rm w}(I-\widetilde{M})$. 
Since $\widetilde{M}$ is an orthogonal projection, the matrix $\Sigma$ is well defined and we have
\begin{align}
\Sigma = (\epsilon_{\rm b}^{-1} - \epsilon_{\rm w}^{-1}){\widetilde M} + \epsilon_{\rm w}^{-1}I. 
\label{eq:pppp}
\end{align}
The 
expression in the right hand side of equation \eqref{eq:gggg} is jointly convex in $W$ and $\Sigma$ \citep[see e.g.][]{Boyd2004}, however the set of matrices $\Sigma$ defined by equation \eqref{eq:pppp}, generated by letting ${\widetilde M} = M\Pi$ vary, is nonconvex, because $M$ takes values on a nonconvex set. 
To address this, \citet{Jacob2009-CLUSTER} relax the constraint on matrix ${\widetilde M}$ to the set $\{0 \preceq {\widetilde M} \preceq I,~\tr {\widetilde M} \leq Q-1\}$. 
This in turn induces the convex constraint set for $\Sigma$
\begin{align*}
{\cal S}_{Q,T} = \left\{\Sigma \in \R^{T \times T} :  \Sigma=\Sigma\trans,~\epsilon_{\rm w}^{-1} I \preceq \Sigma \preceq  \epsilon_{\rm b}^{-1} I, ~\tr \, \Sigma \leq  (\epsilon_{\rm b}^{-1} - \epsilon_{\rm w}^{-1}) (Q-1) + \epsilon_{\rm w}^{-1} T \right\}.
\end{align*}
In summary \citet{Jacob2009-CLUSTER} arrive at the optimization problem
\begin{align}
\min_{W \in \mathbb{R}^{d \times T}} {\bar \calL}(W) + \lambda \Vert W \Pi \Vert_{\rm cl}^2 \label{eqn:cluster-norm-problem-2}
\end{align}
where ${\bar \calL}(W) = \calL(W) + \lambda \epsilon_{\rm m} \, \tr ( W U W\trans)$ and $\|\cdot\|_{\rm cl}$ is the {\em cluster norm} defined by the equation
\begin{align}
\Vert W \Vert_{\rm cl} = \sqrt{\inf_{\Sigma \in  {\cal S}_{Q,T}}  \tr (\Sigma^{-1} W\trans W )}
\label{eq:CN}.
\end{align}


\subsection{The Cluster Norm and the Spectral Box-Norm}\label{sec:cluster-and-theta}
We now discuss the cluster norm in the context of the spectral box-norm. 
\citet{Jacob2009-CLUSTER} state that the cluster norm of $W$ equals what we in this paper have termed the spectral box-norm, with parameters $a = {\epsilon_{\rm w}}^{-1}$, $b = \epsilon_{\rm b}^{-1}$ and 
$c = (T-Q+1) \epsilon_{\rm w}^{-1} + (Q-1) \epsilon_{\rm b}^{-1}$. 
Here we prove this fact. 
Denote by $\lambda_i(\cdot)$ the eigenvalues of a matrix which we write in non increasing order $\lambda_1(\cdot) \geq \lambda_2(\cdot) \geq \ldots \geq \lambda_d(\cdot)$. 
Note that if $\theta_i$ are the eigenvalues of $\Sigma$ then $\theta_i = \lambda_{d-i+1}(\Sigma^{-1})$. 
We have that 
\begin{align*}
 \tr (\Sigma^{-1} W\trans W) 
 \geq \sum_{i=1}^{r} \lambda_{d-i+1}(\Sigma^{-1}) \lambda_i(W\trans W) = \sum_{i=1}^{r} 
 \frac{\sigma_i^2(W)}{\theta_i} 
\end{align*}
where the inequality follows by Lemma \ref{lem:olkin} (stated in the appendix) for $A=\Sigma^{-1}$ and $B=W\trans W \succeq 0$. 
Since this inequality is attained whenever $\Sigma = U {\rm diag}(\theta)U$, where $U$ are the eigenvectors of $W\trans W$, we see that the cluster norm coincides with the spectral box-norm, that is $\Vert W \Vert_{\rm cl} = \|\sigma(W)\|_\Theta$ for $\Theta=\{\theta\in [a,b]^r: \sum_{i=1}^r \theta_i \leq c\}$. 
In light of our observations in Section \ref{sec:matrix-norms}, we also see that the spectral $k$-support norm is a special case of the cluster norm, where we let $a$ tend to zero, $b=1$ and $c=k$, where $k=Q-1$. 
More importantly the cluster norm is a perturbation of the spectral $k$-support norm. 
Moreover, the methods to compute the norm and its proximity operator (cf. Theorems \ref{thm:computation-of-theta-norm} and \ref{thm:prox-of-general-theta-norm}) can directly be applied to the cluster norm using von Neumann's trace inequality (see Theorem \ref{thm:vonneumann} in the appendix).

\subsection{Optimization with Centered Spectral $\Theta$-Norms}\label{sec:centered-norms}
Centering a matrix has been shown to improve learning in other multitask learning problems, for example  \citet{Evgeniou2005} reported improved results using the trace norm. 
It is therefore valuable to address the problem of how to solve a regularization problem of the type
\begin{align}
\min_{W \in \mathbb{R}^{d \times T}} {\bar \calL}(W) + \lambda \Vert W \Pi \Vert_{\Theta}^2 \label{eqn:THETA}
\end{align}
in which the regularizer is applied to the matrix $W \Pi = [w_1- {\bar w},\dots,w_T - {\bar w}]$. 
To this end, 
let $\Theta$ be a bounded and convex subset of $\mathbb{R}_{++}^{r}$ which is invariant under permutation. 
We have already noted that the function defined, for every $W \in \R^{d \times T}$, as
\begin{align*}
\|W\|_\Theta : = \|\sigma(W)\|_\Theta,
\end{align*}
is an orthogonally invariant norm. In particular, problem \eqref{eqn:THETA} includes regularization with the centered cluster norm outlined above. 

Note that right multiplication by the centering operator $\Pi$, is invariant to a translation of the columns of the matrix by a fixed vector, that is, for every $z \in \R^d$, we have $[w_1+z,\dots,w_T+z]\Pi  = W \Pi $. The quadratic term ${\epsilon_{\rm m}} \tr ( W U W\trans)$, which is included in the error, implements square norm regularization of the mean of the tasks, which can help to prevent overfitting. However, in the remainder of this section this term plays no role in the analysis, which equally applies to the case that $\epsilon_{\rm m} =0$.

In order to solve the problem \eqref{eqn:THETA} with a centered regularizer the following lemma is key.
\begin{lemma}
\label{lem:infimum-of-cluster}
Let $r = \min(d,T)$ and let $\Theta$ be a bounded and convex subset of $\mathbb{R}_{++}^{r}$ which is invariant under permutation. For every $W =[w_1,\dots,w_T] \in {\mathbb R}^{d \times T}$, it holds that
\begin{align*}
\|W\Pi\|_{\Theta} = \min_{z \in {\mathbb R}^d} \|[w_1-z,\ldots,w_T-z]\|_{\Theta}.
\end{align*}
\end{lemma}
\begin{proof}
Given the set $\Theta$ we define the set $\Theta^{(T)} = \{\Sigma \in {\bf S}_{++}^T, ~\lambda(\Sigma) \in \Theta\}$ and $\Theta^{(d)} = \{D \in  {\bf S}_{++}^d: \lambda(D) \in \Theta\}$. 
It follows from Lemma \ref{lem:olkin} that

\begin{align*}
\|W\|_\Theta^2 \equiv \|\sigma(W)\|_\Theta^2 
= \inf_{\Sigma \in \Theta^{(T)}} \tr \left( \Sigma^{-1} W\trans W \right)
= \inf_{D \in \Theta^{(d)}} \tr \left( D^{-1} WW\trans \right). 
\end{align*}

Using the second identity and recalling that $W\Pi = [w_1-\bar{w}, \ldots, w_T-\bar{w}]$,  we have that
\begin{align*}
\Vert W \Pi \Vert_{\Theta}^2
&= \inf_{D \in \Theta^{(d)}} \tr ( ( W \Pi)\trans D^{-1} (W \Pi )) \\
&= \inf_{D \in \Theta^{(d)}} \sum_{t=1}^T (w_t- \bar{w})\trans D^{-1} (w_t - \bar{w}) = \inf_{D \in \Theta^{(d)}} \min_{z \in \mathbb{R}^d} \sum_{t=1}^T (w_t- z)\trans D^{-1} (w_t - z)
\end{align*}
where in the last step we used the fact that the quadratic form $\sum_{t=1}^T (w_t- z)\trans D^{-1} (w_t- z)$ is minimized at $z=\bar{w}$. 
The result now follows by interchanging the infimum and the minimum in the last expression and using the definition of the $\Theta$-norm.
\end{proof}

Using this lemma, we rewrite problem \eqref{eqn:cluster-norm-problem-2} as
\begin{align*}
\min_{W \in \mathbb{R}^{d \times T}} \min_{z \in \mathbb{R}^d}{\bar {\cal L}}(W) + \lambda \|[w_1- z,\dots,w_T-z]\|_{\Theta}.
\end{align*}
Letting $\ww_t = w_t -z$, and $\WW = [v_1, \ldots, v_T]$, we obtain the equivalent problem 
\begin{align}
\min_{(\WW,z) \in \mathbb{R}^{d \times T} \times  \mathbb{R}^d}  
{\bar {\calL}}(V+ z 1\trans) + \lambda \Vert \WW \Vert_\Theta^2.\label{eqn:cluster-norm-optimization}
\end{align}
This problem is of the form $f(V,z) + \lambda g(V,z)$, where $g(V,z) =  \|V\|_\Theta$. 
Using this formulation, we can directly apply the proximal gradient method using the proximity operator computation for the vector norm, since $\prox_g(V_0,z_0) = (\prox_{\lambda \|\cdot\|_\Theta}(V_0),z_0)$. 
This observation establishes that, whenever the proximity operator of the spectral $\Theta$-norm is available, we can use proximal gradient methods with minimal additional effort to perform optimization with the corresponding centered spectral $\Theta$-norm. 
For example, this is the case with the trace norm, the spectral $k$-support norm and the spectral box-norm or cluster norm.

\section{Numerical Experiments}
\label{sec:numerics}
\citet{Argyriou2012} demonstrated the good estimation properties of the vector $k$-support norm compared to the Lasso and the {elastic net}.  
In this section, we investigate the matrix norms and report on their statistical performance in matrix completion and multitask learning experiments on simulated as well as benchmark real datasets. 
We also offer an interpretation of the role of the parameters in the box-norm and we empirically verify the improved performance of the proximity operator computation of Algorithm \ref{alg:prox_01} (see Table \ref{table:prox_comparison}).

We compare the spectral $k$-support norm (\emph{k-sup}) and the spectral box-norm (\emph{box}) to the baseline trace norm (\emph{trace}) \citep[see e.g.][]{Argyriou2006,Mazumder2010,Srebro2005,Toh2011}, matrix elastic net (\emph{el.net}) \citep{Li2012} and, in the case of multitask learning, the Frobenius norm (\emph{fr}), which we recall is equivalent to the spectral $k$-support norm when $k=d$.  
As we highlighted in Section \ref{sec:centered-norms}, centering a matrix can lead to improvements in learning.  
For datasets which we expect to exhibit clustering we therefore also apply centered versions of the norms, \emph{c-fr, c-trace, c-el.net, c-k-sup, c-box}.\footnote{As we described in Section \ref{sec:cluster-and-theta}, the cluster norm regularization problem from \citet{Jacob2009-CLUSTER} is equivalent to regularization using the box-norm with a squared $\ell_2$ norm of the mean column vector included in the loss function.  
The centering operator is invariant to constant shifts of the columns, which allows the matrix to have unbounded Frobenius norm when using a centered regularizer.  
The additional quadratic term regulates this effect and can prevent against overfitting. 
We tested the effect of the quadratic term on the centered norms, however the impact on performance was only incremental, and it introduced a further parameter requiring validation.  
On the real datasets in particular, the impact was not significant compared to simple centering, so we do not report on the method below. }

We report test error and standard deviation, matrix rank ($r$) and optimal parameter values for $k$ and $a$, which are determined by validation.
We used a $t$-test to determine the statistical significance of the difference in performance between the regularizers, at a level of $p<0.001$. 

To solve the optimization problem we used an accelerated proximal gradient method (FISTA), \citep[see e.g.][]{Beck2009,Nesterov2007}, using the percentage change in the objective as convergence criterion, with a tolerance of $10^{-5}$ ($10^{-3}$ for real matrix completion experiments).  

As is typical with spectral regularizers such as the trace norm, we found that the spectrum of the learned matrix exhibited a rapid decay to zero.  
In order to explicitly impose a low rank on the final matrix, we included a thresholding step at the end of the optimization.  
For the matrix completion experiments, the thresholding level was chosen by validation. 
Matlab code used in the experiments is available at \url{http://www0.cs.ucl.ac.uk/staff/M.Pontil/software.html}.

\begin{table}[th]
\caption{Comparison of proximity operator algorithms for the $k$-support norm (time in seconds), $k = 0.05 d$. Algorithm 1 is our method, Algorithm 2 is the method in \cite{Argyriou2012}.}
\label{table:prox_comparison}
\begin{center}
\begin{small}
\begin{tabular}{crrrrrr}
\toprule 
$d$ & 1,000 & 2,000 & 4,000 & 8,000 & 16,000 & 32,000  \\ 
\midrule 
Algorithm 1 & 0.0011 & 0.0016 & 0.0026 & 0.0046 & 0.0101 & 0.0181  \\ 
Algorithm 2 & 0.0443 & 0.1567 & 0.5907 & 2.3065 & 9.0080 & 35.6199  \\ 
\bottomrule
\end{tabular}
\end{small}
\end{center}
\end{table}

\subsection{Simulated Data}

{\bf Matrix Completion.} We applied the norms to matrix completion on noisy observations of low rank matrices. 
Each $d \times d$ matrix was generated as $W=A B\trans+E$, where $A,B \in \mathbb{R}^{d\times r}$, $r \ll d$, and the entries of $A$, $B$ and $E$ were set to be i.i.d. standard Gaussian.  
We set $d=100$, $r\in \{5,10\}$ and we sampled uniformly a percentage $\rho \in \{10\%, 10\%, 20\%, 30\%\}$ of the entries for training, and used a fixed 10\% for validation. 
Following \citet{Mazumder2010} the error was measured as 
\begin{align*}
\textrm{error}=\frac{\| w_{\text{true}} - w_{\text{predicted}}\|^2}{\|w_{\text{true}}\|^2},
\end{align*}
and averaged over 100 trials.
The results are summarized in Table \ref{table:mc-synthetic}.
With thresholding, all methods recovered the rank of the true noiseless matrix.  
The spectral box-norm generated the lowest test errors in all regimes, with the spectral $k$-support norm a close second, and both were significantly better than trace and elastic net.

\begin{table}[th]
\caption{Matrix completion on simulated datasets, without (left) and with (right) thresholding. }
\label{table:mc-synthetic}
\centering
\setlength\tabcolsep{5pt}
\begin{minipage}[ht]{0.95\linewidth}
\centering
\begin{small}
\begin{tabular}{llrrrrrrrr}
\toprule
   dataset & norm  & test error & $r$ & $k$ & $a$  & test error   & $r$ & $k$ & $a$  \\
\midrule
rank 5       & trace  & 0.8184 (0.03)  & 20    & -   &-      & 0.7799 (0.04)             & 5     & -   &- \\ 
$\rho$=10\%  & el.net  & 0.8164 (0.03)  & 20    &  -  & -     & 0.7794 (0.04)             & 5     &  -  & -\\
             & k-sup  & 0.8036 (0.03)  & 16    & 3.6 & -     & 0.7728 (0.04)             & 5     & 4.23 & -\\
             & box & 0.7805 (0.03)  & 87    & 2.9 & 1.7e-2  & 0.7649 (0.04)       & 5     & 3.63 & 8.1e-3  \\ 
\midrule
rank 5       & trace   & 0.5764 (0.04) & 22   &-        & -    & 0.5209 (0.04)         & 5   &-        & -     \\  
$\rho$=15\%  & el.net   & 0.5744 (0.04) & 21     &-      & -    & 0.5203 (0.04)             & 5     &-      & -     \\ 
             & k-sup   & 0.5659 (0.03) & 18     & 3.3    &-    & 0.5099 (0.04)           & 5     & 3.25    &-\\ 
             & box  & 0.5525 (0.04)   & 100   & 1.3 & 9e3   & 0.5089 (0.04)       & 5   & 3.36 & 2.7e-3\\ 
\midrule
rank 5       & trace  & 0.4085 (0.03)  & 23    &  -  & -     & 0.3449 (0.02)             & 5     &  -  & - \\
$\rho$=20\%  & el.net  & 0.4081 (0.03)  & 23    &  -  & -     & 0.3445 (0.02)             & 5     &  -  & - \\ 
             & k-sup  & 0.4031 (0.03)  & 21    & 3.1 & -     & 0.3381 (0.02)             & 5     & 2.97 & - \\ 
             & box & 0.3898 (0.03)  & 100   & 1.3 & 9e-3  & 0.3380 (0.02)         & 5     & 3.28 & 1.9e-3    \\ 
\midrule
\midrule
rank 10      & trace  & 0.6356 (0.03)  & 27    &-  & -       & 0.6084 (0.03)             & 10    &-  & -  \\ 
$\rho$=20\%  & el.net  & 0.6359 (0.03)  & 27    & - & -       & 0.6074 (0.03)             & 10    & - & -  \\  
             & k-sup  & 0.6284 (0.03)  & 24    & 4.4 &-      & 0.6000 (0.03)             & 10    & 5.02 &- \\ 
             & box & 0.6243 (0.03)  & 89    & 1.8  & 9e-3  & 0.6000 (0.03)       & 10    & 5.22  & 1.9e-3    \\ 
\midrule
rank 10      & trace  & 0.3642 (0.02)  & 36    &  -  & -     & 0.3086 (0.02)             & 10    &  -  & - \\ 
$\rho$=30\%  & el.net  & 0.3638 (0.02)  & 36    &  -  & -     & 0.3082 (0.02)             & 10    &  -  & - \\ 
             & k-sup  & 0.3579 (0.02)  & 33    & 5.0 &-      & 0.3025 (0.02)             & 10    & 5.13 &-  \\ 
             & box & 0.3486 (0.02)  & 100   & 2.5 & 9e-3  & 0.3025 (0.02)         & 10    & 5.16 & 3e-4      \\ 
\bottomrule
\end{tabular}
\end{small}
\end{minipage}
\end{table}

\vspace{0.1truecm}
\noindent
{\bf Role of Parameters. }In the same setting we investigated the role of the parameters in the box-norm.  
As previously discussed, parameter $b$ can be set to 1 without loss of generality. 
Figure \ref{fig:a-SNR} shows the optimal value of parameter $a$ chosen by validation for varying signal to noise ratios (SNR), keeping $k$ fixed.  
We see that for greater noise levels (smaller SNR), the optimal value for $a$ increases, which further suggests that the noise is filtered out by higher values of the parameter. 
Figure \ref{fig:c-rank} shows the optimal value of $k$ chosen by validation for matrices with increasing rank, keeping $a$ fixed, and using the relation $k=\frac{c-da}{b-a}$.  
We note that as the rank of the matrix increases, the optimal $k$ value increases, which is expected since it is an upper bound on the sum of the singular values. 

\begin{figure}
\begin{minipage}[b]{0.48\linewidth}
\centering
\includegraphics[width=0.99\linewidth]{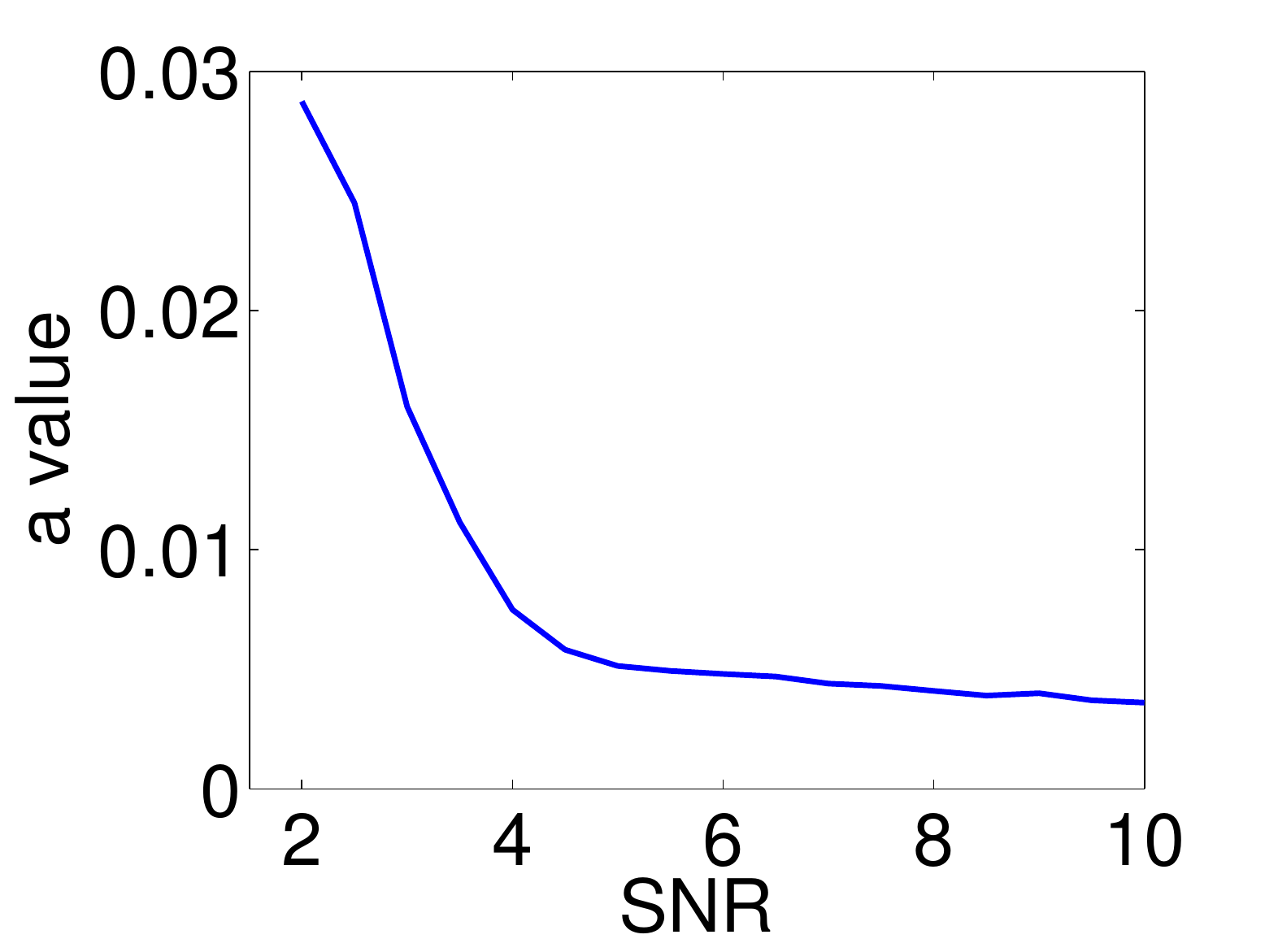}
\caption{Impact of signal to noise ratio on value of $a$.}
\label{fig:a-SNR}
\end{minipage}
\quad
\begin{minipage}[b]{0.48\linewidth}
\centering
\includegraphics[width=0.99\linewidth]{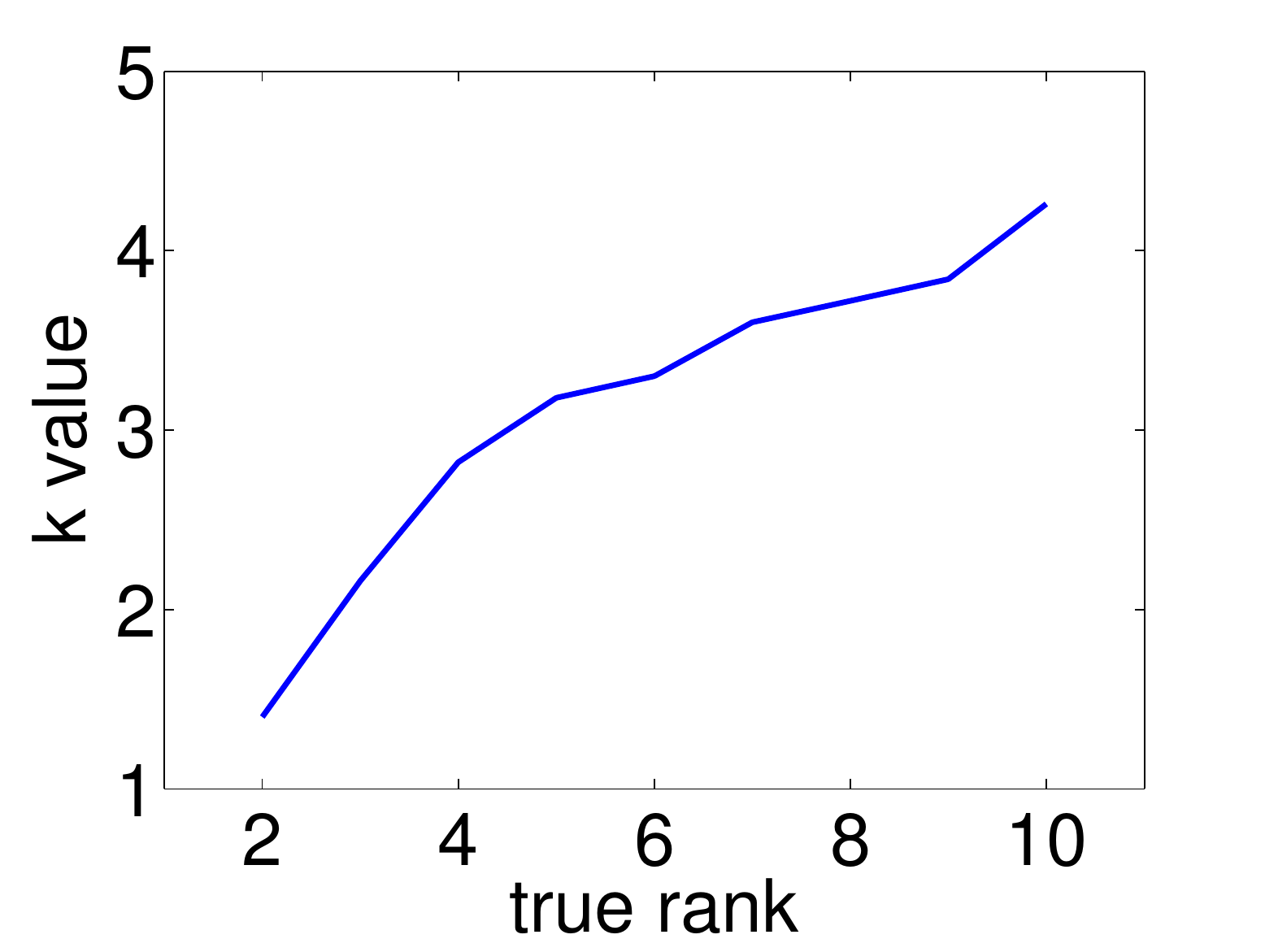}
\caption{Impact of matrix rank on value of $k$.}
\label{fig:c-rank}
\end{minipage}
\end{figure}

\vspace{0.1truecm}
\noindent
{\bf Clustered Learning.} 
We tested the centered norms on a synthetic dataset which exhibited a clustered structure. 
We generated a $100 \times 100$, rank 5, block diagonal matrix, where the entries of each $20\times 20$ block were set to a random integer chosen uniformly in $\{1, \ldots, 10\}$, with additive noise.  
Table \ref{table:simulated-block-diagonal} illustrates the results averaged over 100 runs.  
Within each group of norms, the box-norm and the $k$-support norm outperformed the trace norm and elastic net, and centering improved performance for all norms. 
Figure \ref{fig:block-matrix} illustrates a sample matrix along with the 
solution found using the box and trace norms.

\begin{table}
\caption{Clustered block diagonal matrix, before (left) and after (right) thresholding. }
\label{table:simulated-block-diagonal}
\vspace{0.1truecm}
\centering
\setlength\tabcolsep{5pt}
\begin{minipage}[th]{0.95\linewidth}
\centering
\begin{small}
\begin{tabular}{llrrrrrrrr}
\toprule
dataset & norm & test error    & $r$ & $k$ & $a$   & test error    & $r$ & $k$ & $a$        \\ 
\midrule
$\rho$=10\%    	& trace  		& 0.6529 (0.10) 		& 20 	&    - 		& -  		& 0.6065 (0.10) 	& 5   	& -     & -  	 	\\ 
               	& el.net    	& 0.6482 (0.10)		  	& 20	&    -  	& -  		& 0.6037 (0.10) 	& 5	  	& -  	& -  	       	\\  
              	& k-sup     	& 0.6354 (0.10) 		& 19 	& 2.72     	& -       	& 0.5950 (0.10) 	& 5   	& 2.77 	& - 		 \\
                & box    	& 0.6182 (0.09) 		& 100 	& 2.23 		& 1.9e-2  	& 0.5881 (0.10) 	& 5   	& 2.73 	& 4.3e-3	  \\ 
                & c-trace     	& 0.5959 (0.07) 		& 15 	& - 		& -   		& 0.5692 (0.07) 	& 5   	& - 	& -    	 	\\ 
                & c-el.net    	& 0.5910 (0.07) 		& 14 	& - 		& -  		& 0.5670 (0.07) 	& 5   	& -    	& -   	 	\\ 
                & c-k-sup    	& 0.5837 (0.07) 		& 14 	& 2.03   	& -      	& 0.5610 (0.07) 	& 5   	& 1.98 	& -    	  \\  
                & c-box    	& 0.5789 (0.07) 		& 100 	&  1.84	 	& 1.9e-3 	& 0.5581 (0.07) 	& 5   	& 1.93 	& 9.7e-3      \\      
\midrule
\midrule
$\rho$=15\%    	& trace       	& 0.3482 (0.08) 		& 21 	& - 	  	& -   		& 0.3048 (0.07) 	& 5		&    -      & -              \\ 
               	& el.net       	& 0.3473 (0.08) 		& 21	& -  		& -    		& 0.3046 (0.07) 	& 5		&    -      & -          \\
               	& k-sup       	& 0.3438 (0.07)  		& 21	& 2.24     	& -    		& 0.3007 (0.07) 	& 5		& 2.89      &-         \\ 
               	& box      	& 0.3431 (0.07) 		& 100	& 2.05 		& 8.7e-3  	& 0.3005 (0.07) 	& 5		& 2.57 		& 1.3e-3      \\ 
               	& c-trace     	& 0.3225 (0.07) 		& 19 	& - 		& -     	& 0.2932 (0.06) 	& 5   	& -         & - 		\\ 
               	& c-el.net     	& 0.3215 (0.07) 		& 18 	& - 		& -   		& 0.2931 (0.06) 	& 5   	& -         & - 		\\  
               	& c-k-sup     	& 0.3179 (0.07) 		& 18 	& 1.89		& -        	& 0.2883 (0.06) 	& 5   	& 2.36      & -         \\ 
               	& c-box   	& 0.3174 (0.07) 		& 100 	& 1.90 		& 2.2e-3  	& 0.2876 (0.06) 	& 5   	& 1.92  	& 3.8e-3   \\ 
\bottomrule
\end{tabular}
\end{small}
\end{minipage}
\end{table}

\begin{figure}
\caption{Clustered matrix and recovered solutions. From left to right: true, noisy, trace norm, box-norm}
\label{fig:block-matrix}
\centering
\begin{minipage}{0.24\linewidth}
\centering
  \includegraphics[width=1.3\linewidth]{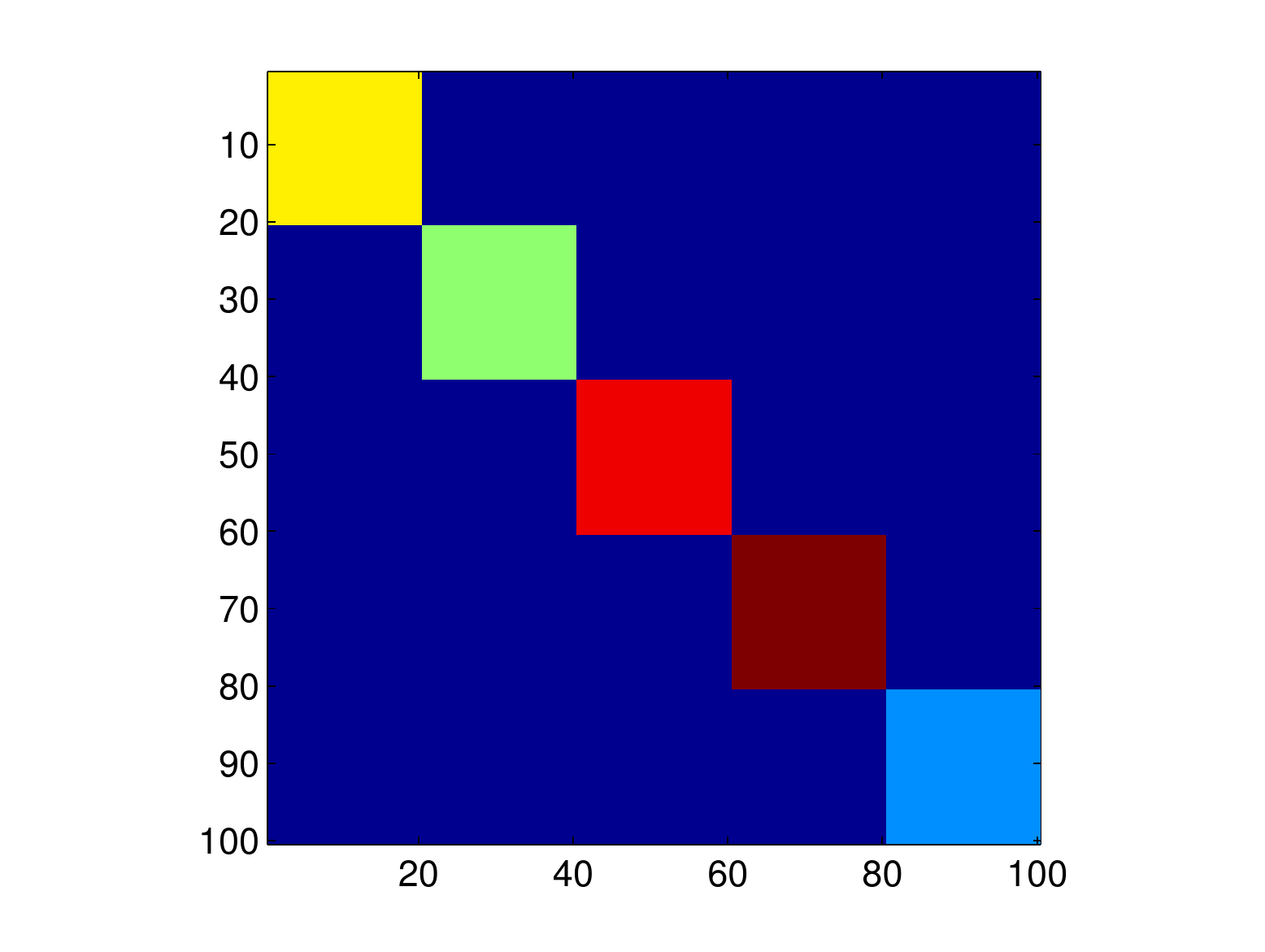}
\end{minipage}
\begin{minipage}{0.24\linewidth}
\centering
  \includegraphics[width=1.3\linewidth]{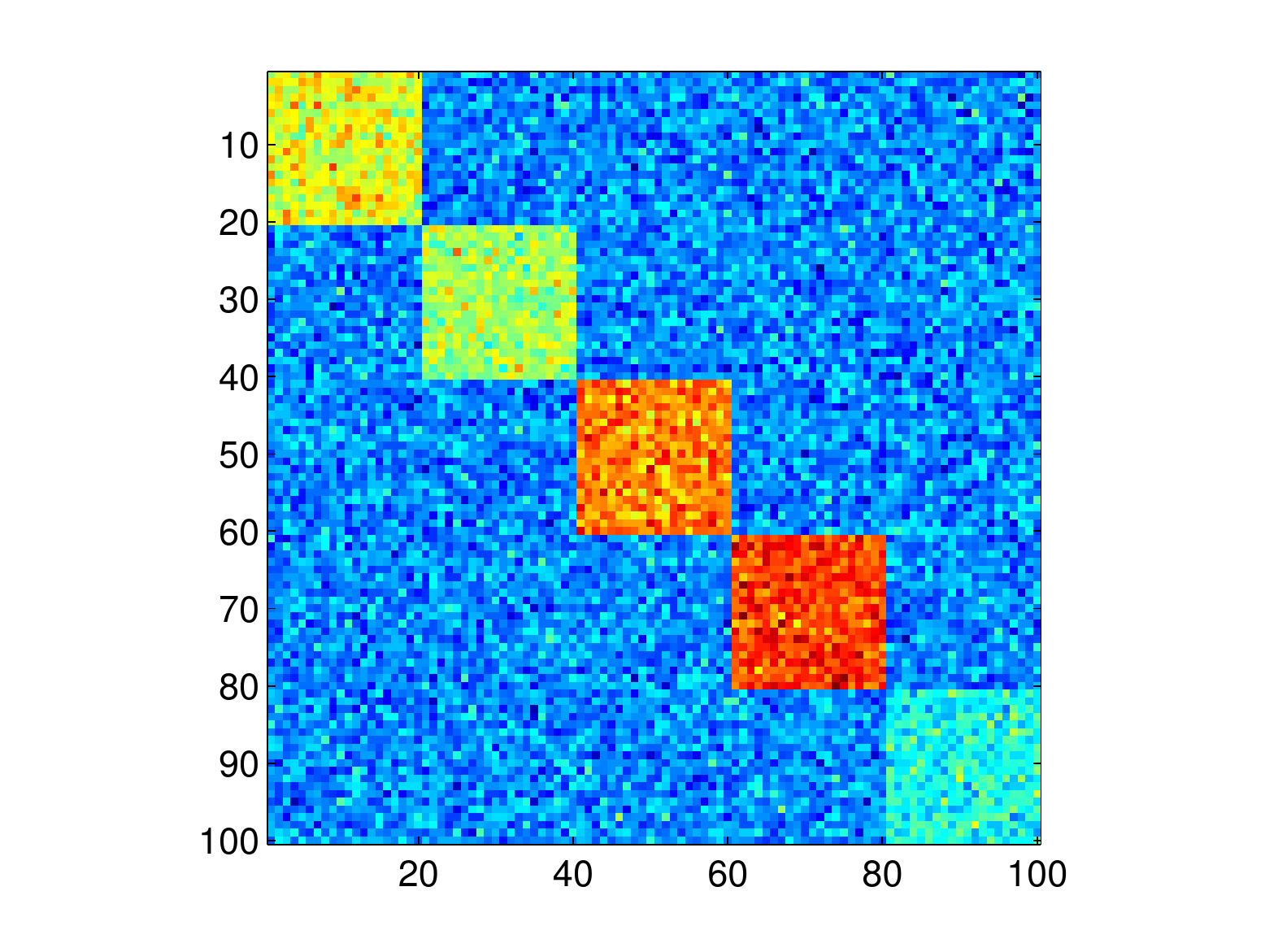}
\end{minipage}
\begin{minipage}{0.24\linewidth}
\centering
  \includegraphics[width=1.3\linewidth]{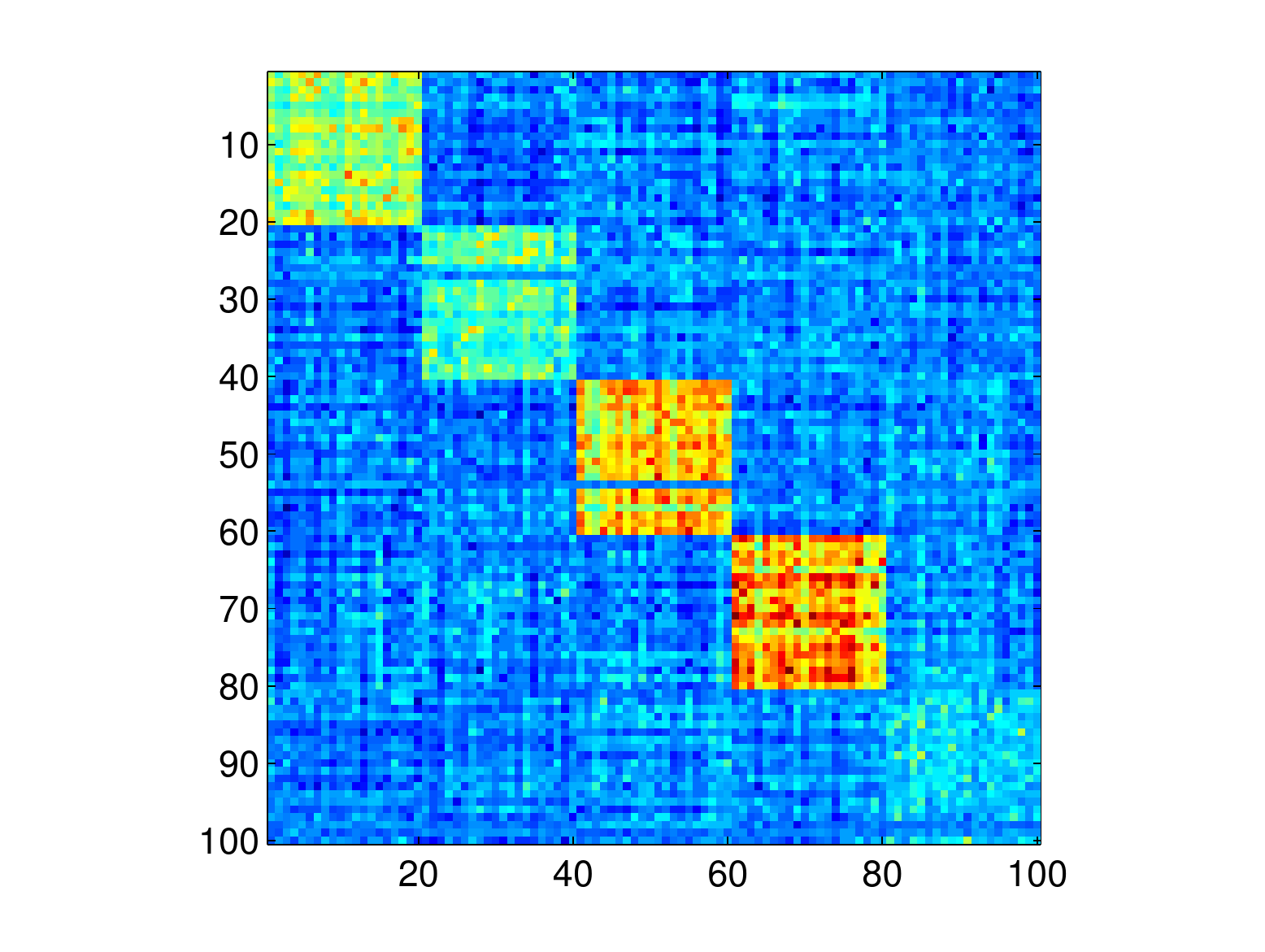}
\end{minipage}
\begin{minipage}{0.24\linewidth}
\centering
  \includegraphics[width=1.3\linewidth]{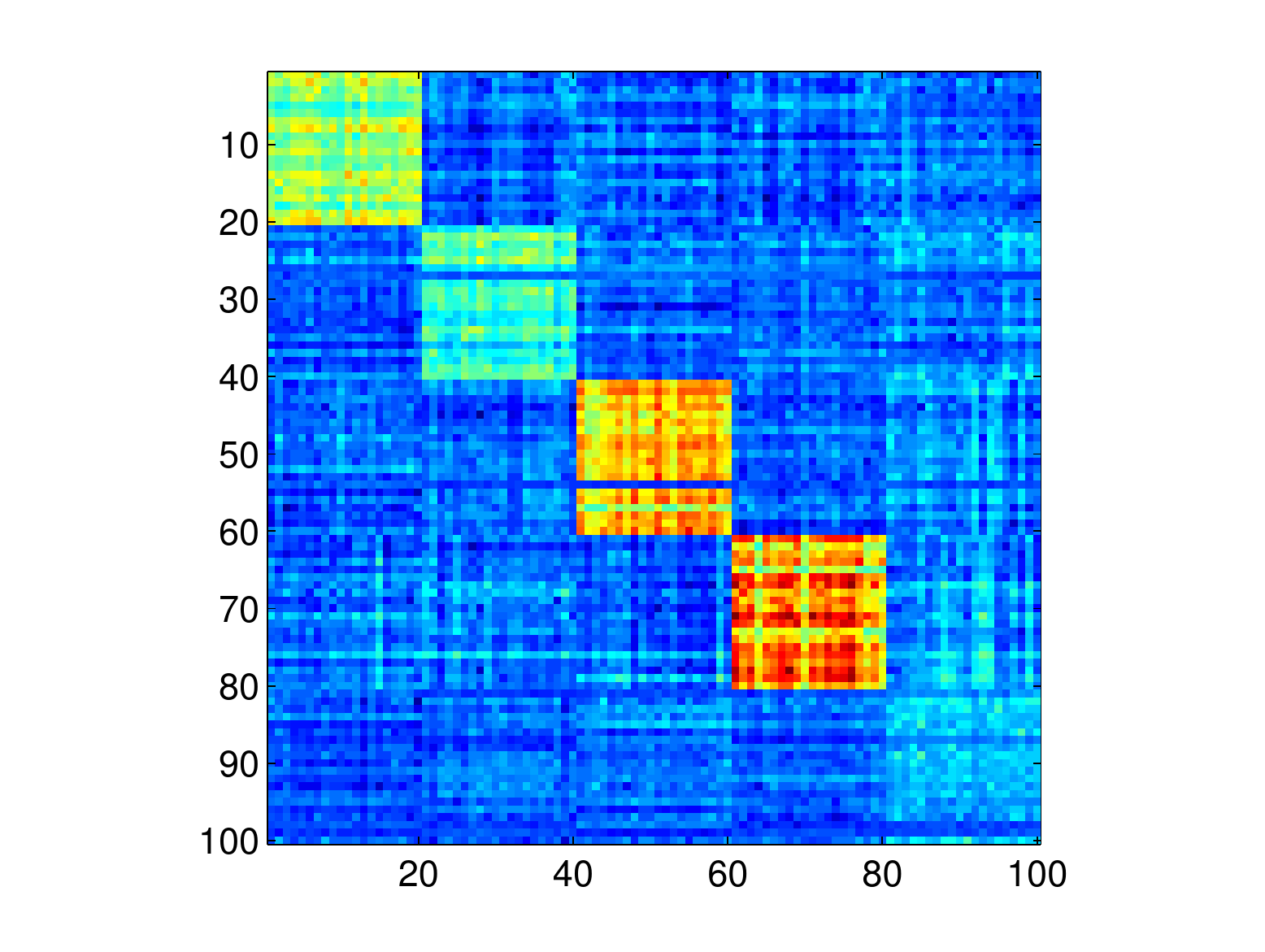}
\end{minipage}
\end{figure}

\subsection{Real Data}
{\bf Matrix Completion (MovieLens and Jester). }
In this section we report on the performance of the norms on real datasets.
We observe a subset of the (user, rating) entries of a matrix and the task is to predict the unobserved ratings, with the assumption that the true matrix is low rank (or approximately low rank). 
In the first instance we considered the MovieLens datasets\footnote{MovieLens datasets are available at {\em http://grouplens.org/datasets/movielens/}.}.
These consist of user ratings of movies, the ratings are integers from 1 to 5, and all users have rated a minimum number of 20 films. 
Specifically we considered the following datasets: 
\begin{itemize}
\item \emph{MovieLens 100k}: 943 users and 1,682 movies, with a total of 100,000 ratings;
\item \emph{MovieLens 1M}: 6,040 users and 3,900 movies, with a total of 1,000,209 ratings.  
\end{itemize}
We also considered the Jester \footnote{Jester datasets are available at {\em http://goldberg.berkeley.edu/jester-data/}.} datasets, which consist of user ratings of jokes, where the ratings are real values from $-10$ to $10$:
\begin{itemize}
\item \emph{Jester 1}: 24,983 users and 100 jokes, all users have rated a minimum of 36 jokes;  
\item \emph{Jester 2}: 23,500 users and 100 jokes, all users have rated a minimum of 36 jokes;
\item \emph{Jester 3}: 24,938 users and 100 jokes, all users have rated between 15 and 35 jokes. 
\end{itemize}
Following \citet{Toh2011}, for MovieLens we uniformly sampled $\rho=50\%$ of the available entries for each user for training, and for Jester 1, Jester 2 and Jester 3 we sampled 20, 20 and 8 ratings per user respectively, and we again used 10\% for validation.
The error was measured as normalized mean absolute error, 
\begin{align*}
\textrm{NMAE}=\frac{\Vert w_{\text{true}}-w_{\text{predicted}}\Vert^2}{\#\text{observations}/(r_{\max} -r_{\min})},
\end{align*}
where $r_{\max}$ and $r_{\min}$ are upper and lower bounds for the ratings \citep{Toh2011}, 
averaged over 50 runs. 
The results are outlined in Table \ref{table:mc-real}.
In the thresholding case, the spectral box-norm and the spectral $k$-support norm showed the best performance, and in the absence of thresholding, the spectral $k$-support norm showed slightly improved performance.
Comparing to the synthetic datasets, this suggests that the parameter $a$ did not provide any benefit in the absence of noise. 
We also note that without thresholding our results for trace norm regularization on MovieLens 100k agreed with those in \citet{Jaggi2010}.

\begin{table}
\caption{Matrix completion on real datasets, without (left) and with (right) thresholding.}
\label{table:mc-real}
\centering
\setlength\tabcolsep{4pt}
\begin{minipage}[th]{0.95\linewidth}
\centering
\begin{small}
\begin{tabular}{llrrrrrrrr}
\toprule
 dataset    & norm  & test error    & $r$   & $k$ & $a$ & test error    & $r$   & $k$ & $a$\\
\midrule
MovieLens   & trace    & 0.2034    & 87    &   -   &-  & 0.2017    & 13    &   -   &-          \\
100k        & el.net    & 0.2034    & 87    &    -  &-   & 0.2017    & 13    &   -   &-           \\  
$\rho=50\%$ & k-sup    & 0.2031    & 102   & 1.00  &-   & 0.1990    & 9     & 1.87  &- \\               
            & box   & 0.2035    & 943   & 1.00  & 1e-5 & 0.1989    & 10    & 2.00  & 1e-5\\               
\midrule
MovieLens   & trace    & 0.1821    & 325   &   -   &-  & 0.1790    & 17    &   -   &-        \\ 
1M          & el.net    & 0.1821    & 319   &   -   &-   & 0.1789    & 17    &   -   &-        \\
$\rho=50\%$ & k-sup    & 0.1820    & 317   & 1.00  &-  & 0.1782    & 17    & 1.80  &-  \\ 
            & box   & 0.1817    & 3576  & 1.09  & 3e-5  & 0.1777    & 19    & 2.00  & 1e-6 \\             
\midrule            
Jester 1    & trace    & 0.1787    & 98    &    -  &-  & 0.1752    & 11    &  -    &-     \\
20 per line & el.net    & 0.1787    & 98    &   -   &- & 0.1752    & 11    &  -    &-          \\
            & k-sup    & 0.1764    & 84    & 5.00  &-  & 0.1739    & 11    & 6.38  &-     \\
            & box   & 0.1766    & 100   & 4.00  & 1e-6   & 0.1726   & 11    & 6.40  & 2e-5 \\           
\midrule       
Jester2     & trace    & 0.1767      & 98    &    -&-  & 0.1758    & 11    &    -  &-          \\  
20 per      & el.net  & 0.1767  & 98    &   -&-      & 0.1758    & 11    &   -   &-     \\
  line      & k-sup    & 0.1762    & 94      & 4.00 &-  & 0.1746    & 11    & 4.00  &-      \\
            & box   & 0.1762      & 100     & 4.00 & 2e-6  & 0.1745    & 11    & 4.50  & 5e-5 \\            
\midrule            
Jester 3    & trace    & 0.1988    & 49    &   -   &-   & 0.1959    & 3     &   -   &-        \\ 
8 per line  & el.net    & 0.1988    & 49    &   -   &-       & 0.1959    & 3     &   -   &-    \\  
            & k-sup    & 0.1970    & 46    & 3.70  &-    & 0.1942    & 3     & 2.13  &-   \\ 
            & box   & 0.1973    &  100  & 5.91  & 1e-3  & 0.1940    & 3     & 4.00  & 8e-4 \\ 
\bottomrule
\end{tabular}
\end{small}
\end{minipage}
\end{table}

\vspace{0.1truecm}
\noindent
{\bf Multitask Learning (Lenk and Animals with Attributes). }
In our final set of experiments we considered two multitask learning datasets, where we expected the data to exhibit clustering.  
The \emph{Lenk personal computer} dataset \citep{Lenk1996} consists of 180 ratings of 20 profiles of computers characterized by 14 features (including a bias term).  
The clustering is suggested by the assumption that users are motivated by similar groups of features. 
We used the root mean square error of true vs. predicted ratings, normalised over the tasks, averaged over 
100 runs.
We also report on the Frobenius norm, which in the multitask learning framework corresponds to independent task learning.
The results 
are outlined in Table \ref{table:lenk-mtl}.  
The centered versions of the spectral $k$-support norm and spectral box-norm outperformed the other penalties in all regimes.
Furthermore, the results clearly indicate the importance of centering, as discussed for the trace norm in \citet{Evgeniou2007}.

\begin{table}
\caption{Multitask learning clustering on Lenk dataset. }
\label{table:lenk-mtl}
\vspace{0.1truecm}
\centering
\setlength\tabcolsep{8pt}
\begin{minipage}[th]{0.95\linewidth}
\centering
\begin{small}
\begin{tabular}{lccc}
\toprule
    norm  & test error  & $k$ & $a$  \\ 
\midrule
 fr     		& 3.7931 (0.07)     		&   -  	& -      		\\
  trace    & 1.9056 (0.04)     	 	&  -    	&-      	     	\\ 
          el.net      	& 1.9007 (0.04)        	& -  	&   - 		    \\  
      	 k-sup  		& 1.8955 (0.04)     	 	& 1.02  	& - 		 		\\ 
      	 box  		& 1.8923 (0.04)      	& 1.01 	& 5.5e-3 		\\ 
         c-fr     	& 1.8634 (0.08)    	 	& -     	&-    	   		\\      
        c-trace     	& 1.7902 (0.03)   	 	& -     	&-    	   		\\ 
         c-el.net     	& 1.7897 (0.03)    	 	& -     	&-     	  		\\ 
         c-k-sup    	& 1.7777 (0.03)     	 	& 1.89 	&-    	     	\\
         c-box    	& 1.7759 (0.03)    	 	& 1.12  	& 8.6e-3 \\  
\bottomrule
\end{tabular}
\end{small}
\end{minipage}
\end{table}

The \emph{Animals with Attributes} dataset \citep{Lampert2009} consists of 30,475 images of animals from 50 classes.  
Along with the images, the dataset includes pre-extracted features for each image. 
The dataset has been analyzed in the context of multitask learning. 
We followed the experimental protocol from \citep{Kang2011}, however we used an updated feature set, and we considered all 50 classes. 
Specifically, we used the DeCAF feature set provided by \citet{Lampert2009} rather than the SIFT bag of word descriptors.  
These updated features were obtained through a deep convolutional network and represent each image by a 4,096-dimensional vector \citep{Donahue2014}. 
As the smallest class size is 92
we selected the first $n=92$ examples of each of the $T=50$ classes, 
used PCA (with centering) on the resulting data matrix to reduce dimensionality ($d=1,718$)
retaining a variance of 95\%, and obtained a dataset of size $4,600 \times 1,718$. 
For each class the examples were split into training, validation and testing datasets, with a split of 50\%, 25\%, 25\% respectively, and we averaged the performance over 50 runs. 

We used the logistic loss, yielding the error term
$$
{\cal L}(W) = \sum_{t=1}^T \sum_{i=1}^{Tn} \log \left( 1+ \exp(-y_{t,i} \langle w_t, x_i\rangle )\right)
$$
where $W=[w_1, \ldots, w_T]$, $x_1, \ldots, x_{Tn}$ are the inputs and $y_{t,i} = 1$ if $x_i$ is in class $t$, and $y_{t,i} =-1$ otherwise.

The predicted class for testing example $x$ was $\argmax_{t=1}^T \langle  w_t, x \rangle$ and the accuracy was measured as the percentage of correctly classified examples, also known as multi-class error. 
The results without centering are outlined in Table \ref{table:lenk-awa}. 
The corresponding results with centering showed the same relative performance, but worse overall accuracy, which is reasonable as the data is not expected to be clustered, and we omit the results here.

The spectral $k$-support and box-norms gave the best results, outperforming the Frobenius norm and the matrix elastic net, which in turn outperformed the trace norm.
We highlight that in contrast to the Lenk experiments, the Frobenius norm, corresponding to independent task learning, was competitive.  
Furthermore, the optimal values of $k$ for the spectral $k$-support norm and spectral box-norm were high (38 and 33, respectively) relative to the maximum rank of 50, corresponding to a relatively high rank solution. 
The spectral $k$-support norm and spectral box-norm nonetheless outperformed the other regularizers. 
Notice also that the spectral $k$-support norm requires the same number of parameters to be tuned as the matrix elastic net, which suggests that it somehow captures the underlying structure of the data in a more appropriate manner. 

We finally note as an aside that using the SIFT bag of words descriptors provided by \citet{Lampert2009}, which represent the images as a $2,000$-dimensional histogram of local features, we replicated the results for independent task learning (Frobenius norm regularization) and single-group learning (trace norm regularization) of \citet{Kang2011} for the subset of 20 classes considered in their paper.

\begin{table}
\caption{Multitask learning clustering on Animals with Attributes dataset, no centering. }
\label{table:lenk-awa}
\vspace{0.1truecm}
\centering
\setlength\tabcolsep{8pt}
\begin{minipage}[th]{0.95\linewidth}
\centering
\begin{small}
\begin{tabular}{lccc}
\toprule
  norm  & test error  & $k$ & $a$  \\ 
\midrule
\midrule
 fr     	& 38.3428 (0.74)     		&   -  	& -      		\\ 
  	 tr    & 37.4285 (0.76)     	 	&  -    	&-      	     	\\  
      	el.net      	& 38.2857 (0.73)        	& -  	&   - 		    \\ 
      	 k-sup  		& 38.8571 (0.71)     	 	& 37.8  	& - 		 		\\ 
      	 box  		& 38.9100 (0.65)      	& 32.8 	& 2.1e-2 		\\ 
\bottomrule
\end{tabular}
\end{small}
\end{minipage}
\end{table}

\section{Extensions}
\label{sec:extensions}
In this section we outline a number of extensions to topics in this paper.

\subsection{$k$-Support $p$-Norms}
A natural extension of the $k$-support norm follows by applying a $p$-norm, rather than the Euclidean norm, in the infimum convolution definition of the $k$-support norm.  
In the dual norm, we then obtain the corresponding $q$-norm, where $\frac{1}{p}+\frac{1}{q}=1$.
\begin{definition}\label{def:k-sup-p-norm}
The $k$-support $p$-norm is defined for $w \in \R^d$ as 
\begin{align}
\Vert w \Vert_{(k,p)} &= \inf \left\{ \sum_{g \in \G_k} \Vert v_g \Vert_p : \supp(v_g) \subseteq g, \sum_{g \in \G_k} v_g = w \right\}. \label{eqn:k-sup-p-norm}
\end{align}
\end{definition}

The following corollary follows along the same lines as the proof of Proposition \ref{prop:geo2} in the appendix.  
\begin{corollary}\label{prop:k-sup-p-norm-and-dual}
The $(p,k)$-support norm is well defined and its unit ball is the convex hull of the set $\{w \in \R^d: \|w\|_p \leq 1,~ {\rm card}(w) \leq k\}$. Furthermore, its dual norm is given by 
\begin{align*}
\Vert u \Vert_{*,(k,q)} = \left(\sum_{i=1}^k (\vert u \vert^{\downarrow}_i)^q\right)^\frac{1}{q},~~~ u \in \R^d.
\end{align*}
\end{corollary}

We discuss the special cases $p\in \{1,2,\infty\}$.
The case $p=2$ is the $k$-support norm of \cite{Argyriou2012} discussed above. 
For $p=1$ we have $\Vert u \Vert_{*,(k,q)} =\|u\|_\infty$, hence the $(k,1)$-support norm coincides with the $\ell_1$ norm for every $k \in \NN{d}$. 
{The case $p=\infty$ is more interesting; specifically the dual norm is the well-known Ky-Fan norm \citep[see e.g.][]{Bhatia1997}.}
\begin{align*}
\Vert u \Vert_{*,(k,\infty)} = \sum_{i=1}^k |u|^{\downarrow}_i.
\end{align*}
Using the fact that the primal norm is the dual of the dual, we obtain by a direct computation that
\begin{align*}
\Vert w \Vert_{(k,\infty)} = \max\left(\|w\|_\infty,\frac{1}{k} \|w\|_1\right).
\end{align*}

It is clear that the $(p,k)$-support norm is a symmetric gauge function. Hence we we can define the spectral $(p,k)$-support norm as $\Vert W \Vert_{(k,p)} = \Vert \sigma(W) \Vert_{(k,p)}$, for $W \in \R^{d \times T}$. Since the dual of any orthogonally invariant norm is given by $\|\cdot\|_* = \Vert \sigma(\cdot) \Vert_*$ \citep[see e.g.][]{Lewis1995}, we conclude that the dual spectral $(k,p)$-support norm is given by $\Vert U \Vert_{*,(k,p)} = \Vert \sigma(U) \Vert_{*,(k,p)}$, for every $U \in \R^{d \times T}$.
Furthermore, the unit ball of the spectral $(p,k)$-support norm is equal to the convex hull of the set $\{W \in \R^{d \times T}: \rank(W) \leq k, \, \Vert \sigma(W) \Vert_{p} \leq 1 \}$.

\subsection{Kernels}
The ideas discussed in this paper can be used in the context of multiple kernel learning in a natural way \citep[see e.g.][and references therein]{MicPon07}. Let $K_j$, $j \in \N_s$, be prescribed reproducing kernels on a set $X$ , and $H_j$ the corresponding reproducing kernel Hilbert spaces with norms $\|\cdot\|_{j}$. We
consider the problem
$$ \min \left\{ \sum_{i=1}^n \ell\Big(y_i,\sum_{\ell=1}^s f_\ell(x_i)\Big)
+ \lambda \big\|\big(\|f_1\|_{1},\dots,\|f_s\|_{s}\big)\big\|_\Theta^2 : f_1 \in H_1,\dots,f_s \in H_s\right\}.
$$
The choice $\Theta  = \{\theta \in {\mathbb R}^d: 0 < \theta_i \leq 1,~\sum_{i=1}^d \theta_i \leq k\}$, when $k \leq s$, is particularly interesting. It gives rise to a version of multiple kernel learning in which at least $k$ kernels are employed.

\subsection{Rademacher complexity}
We briefly comment on the Rademacher complexity of the spectral $k$-support norm, namely
$$
\E \sup_{\|W\|_{(k)} \leq 1} \frac{1}{Tn} \sum_{t=1}^T \sum_{i=1}^n \epsilon_i^t \lb w_t,x_i^t\rb 
$$
where the expectations is taken with respect to  i.i.d.~Rademacher random variables $\epsilon_i^t$, $i \in \N_n,~t \in \N_T$ and the $x_i^t$ are either prescribed or random datapoints associated with the different regression tasks.
 The Rademacher complexity can be used to derive uniform bounds on the estimation error and excess risk bounds \citep[see][for a discussion]{Bartlett2002,Koltchinskii2002}.
Although a complete analysis is beyond the scope of the present paper, we remark that the Rademacher complexity of the unit ball of the spectral $k$-support is a factor of $\sqrt{k}$ larger than the Rademacher complexity bound for the trace norm provided in \citep[Proposition 6][]{Maurer2013}. 
This follows from the fact that the dual spectral $k$-support norm is bounded by $\sqrt{k}$ times the operator norm.
Of course the unit ball of the spectral $k$-support norm contains the unit ball of the trace norm, so the associated excess risk bounds need to be compared with care.




\section{Conclusion}
\label{sec:conclusion}
We studied the family of box-norms, and showed that the $k$-support norm belongs to this family. 
We noted that these can be naturally extended from the vector to the matrix setting. 
We also provided a connection between the $k$-support norm and the cluster norm, which essentially coincides with the spectral box-norm. 
We further observed that the cluster norm is a perturbation of the spectral $k$-support norm, and we were able to compute the norm and the proximity operator of the squared norm. 
We also provided a method to solve regularization problems using centered versions of the norms and we considered a number of extensions to the box-norm framework. 

Our experiments indicate that the spectral box-norm and $k$-support norm consistently outperform the trace norm and the matrix elastic net on various matrix completion problems. 
Furthermore, we studied the application of centering to clustering problems in multitask learning, and found that this improved performance.  
With a single parameter, compared to two for the spectral box-norm, and three for the cluster norm, our results suggest that the spectral $k$-support norm represents a powerful yet straightforward alternative to the trace norm for low rank matrix learning.
In future work we would like to complete the analysis of the Rademacher complexity for the norms in this paper, and derive associated statistical oracle inequalities. 
We would also like to investigate the family of $\Theta$-norms for more general parameter sets.


\acks{We would like to thank Andreas Maurer, Charles Micchelli and especially Andreas Argyriou for useful discussions.  
This work was supported in part by EPSRC Grant EP/H027203/1.}

\appendix{}

\section{}
\label{sec:app-aux}

In this appendix, we discuss some auxiliary results which are used in the main body of the paper.

Let $X$ be a finite dimensional vector space. 
Recall that a subset $C$ of $X$ is called {\em balanced} if $\alpha C \subseteq C$ whenever $\vert \alpha \vert \leq 1$. 
Furthermore, $C$ is called {\em absorbing} if for any $x \in X$, $x \in \lambda C$ for some $\lambda >0$.  
\begin{lemma}\label{lem:minkowski-bounded}
Let $C\subseteq X$ be a bounded, convex, balanced, and absorbing set.  
The Minkowski functional $\mu_C$ of $C$, defined, for every $w \in X$, as
\begin{align*}
\mu_C(w) = \inf \left\{\lambda : \lambda >0, ~\frac{1}{\lambda} w \in C\right\}
\end{align*}
is a norm on $X$. 
\end{lemma}

\begin{proof}
We give a direct proof that $\mu_C$ satisfies the properties of a norm.  
See also e.g. \cite[\S 1.35]{Rudin1991} for further details.
Clearly $\mu_C(w) \geq 0$ for all $w$, and $\mu_C(0)=0$. Moreover, as $C$ is bounded,
$\mu_C(w) > 0$ whenever $w\ne 0$. 

Next we show that $\mu_C$ is one-homogeneous. For every $\alpha \in \R$, $\alpha \ne 0$, let $\sigma = {\rm sign}(\alpha)$ and note that
\begin{align*}
\mu_C(\alpha w) &= \inf \left\{\lambda>0 : \frac{1}{\lambda} \alpha w \in C \right\}\\
&= \inf \left\{\lambda>0 : \frac{|\alpha|}{\lambda}  \sigma w \in C \right\}\\
&= \vert \alpha \vert \inf \left\{\lambda>0 : \frac{1}{\lambda} 	w \in \sigma C \right\}\\
&= \vert \alpha \vert \inf \left\{\lambda>0 : \frac{1}{\lambda} 	w \in  C \right\}= \vert \alpha \vert \mu_C(w),
\end{align*}
where we have made a change of variable and used the fact that $\sigma C = C$.

Finally, we prove the triangle inequality. For every $v,w \in X$, if $v/\lam \in C$ and $w/\mu \in C$ then setting $\gamma =  \lam/(\lam+\mu)$, we have
$$
\frac {v+w}{\lam +\mu} =\gamma \frac{v}{\lam} + (1-\gamma)  \frac{w}{\mu}
$$
and since $C$ is convex, then $\frac {v+w}{\lam +\mu} \in C$. We conclude that $\mu_C(v+w) \leq \mu_C(v) + \mu_C(w)$. The proof is completed. 
\end{proof}
Note that for such set $C$, the unit ball of the induced norm $\mu_C$ is $C$. 
Furthermore, if $\|\cdot\|$ is a norm then its unit ball satisfies the hypotheses of Lemma \ref{lem:minkowski-bounded}. 


Using this lemma we can prove Proposition \ref{prop:geo2}.
\vspace{.2truecm}

\begin{proof}{\bf of Proposition \ref{prop:geo2}}
Let $A_{\ind} = \{w \in \R^d: \|w\|_{\gamma^\ind} \leq 1,~\supp(w) \subseteq \supp(\gamma^\ind)\}$, and define 
$$
C ={\rm co} \bigcup\limits_{\ell=1}^m A_\ell.
$$
Note that $C$ is bounded and balanced, since each set $A_\ell$ is so. Furthermore, the hypothesis that $\sum_{\ell=1}^m \gamma^\ell > 0$ ensures that $C$ is absorbing. 
Hence, by Lemma \ref{lem:minkowski-bounded} the Minkowski functional $\mu_C$ defines a norm. 
We rewrite $\mu_C(w)$ as
$$
\mu_C(w) =  
\inf \left\{ \lambda : \lambda > 0,~w = \lambda \sum_{\ind=1}^m \alpha_{\ind} z_\ind,~z_\ind \in A_{\ind},~\alpha \in \Delta^{m-1}\right\} 
$$
where the infimum is over $\lambda$, the vectors $z_\ind \in\R^d$ and the vector $\alpha=(\alpha_1,\dots,\alpha_m)$, and recall $\Delta^{m-1}$ denotes the unit simplex in $\R^m$. 

The rest of the proof is structured as follows. We first show that $\mu_C(w)$ coincides with the right hand side of equation \eqref{Alabel}, which we denote by $\nu(w)$. Then we show that $\|w\|_\Theta = \mu_C(w)$ by observing that both norms have the same dual norm. 

Choose any vectors 
$v_1,\dots,v_m \in \R^d$ 
which satisfies the constraint set in the right hand side of \eqref{Alabel} and set $\alpha_\ell = \Vert v_{\ind} \Vert_{\gamma^{\ind}} /(\sum_{k=1}^m \Vert v_{k} \Vert_{\gamma^k})$ and $z_\ell = v_\ind /\|v_\ind\|_{\gamma^\ell}$. We have
$$
w = \sum_{\ind=1}^m v_{\ind} = \Big( \sum_{k=1}^m \Vert v_k \Vert_{\gamma^k} \Big) \sum_{\ind=1}^m 
\alpha_{\ind} z_{\ind}.
$$
This implies that $\mu_C(w) \leq \nu(w)$. Conversely, if $w = \lambda \sum_{\ell=1}^m \alpha_\ell z_\ell$ for some $z_\ell \in A_\ell$ and $\alpha\in \Delta^{m-1}$, then letting $v_\ind = \lambda \alpha_\ell z_\ell$ we have
\begin{align*}
 \sum_{\ind = 1}^m  \Vert v_{\ind} \Vert_{\gamma^{\ind}} 
&= \sum_{\ind = 1}^m  \Vert  \lambda \alpha_{\ind} z_{\ind} \Vert_{\gamma^{\ind}} 
=  \lambda  \sum_{\ind = 1}^m \alpha_{\ind}\Vert  z_{\ind} \Vert_{\gamma^{\ind}} 
\leq \lambda.
\end{align*}

Next, 
we show that both norms have the same dual norm. 
We noted in Proposition \ref{prop:theta-is-norm} that the dual norm of $\|\cdot\|_\Theta$ takes the form \eqref{eqn:theta-dual}. When $\Theta$ is the interior of ${\rm co} \{\gamma^1,\dots,\gamma^m\}$, this can be written as
\begin{align*}
\|u\|_{*,\Theta} = \sup_{\theta \in \Theta} \sqrt{\sum_{i=1}^d \theta_i u_i^2} = \max_{\ind=1}^m  \sqrt{\sum_{i=1}^d \gamma^\ind_i u_i^2}.
\end{align*}
We now compute the dual of the norm $\mu_C$, 
\begin{align}
\max_{w\in C}  \lb w,u\rb  = \max \left\{ \lb w,u\rb  : w\in \cup_{\ind=1}^m A_{\ind} \right\}  =  \max_{\ind=1}^m \max_{w \in A_{\ind}} \lb w,u\rb =\max_{\ind=1}^m \sqrt{\sum_{i=1}^d \gamma^\ind_i u_i^2}. \label{eqn:dual-is-max}
\end{align}
It follows that the norms share the same dual norm, hence $\mu_C(\cdot)$ coincides with $\|\cdot\|_\Theta$.
\end{proof}

The above proof reveals that the unit ball of the dual norm of $\Vert \cdot \Vert_{\Theta}$ is given by an intersection of ellipsoids in $\R^d$.  
Indeed equation \eqref{eqn:dual-is-max} provides that
\begin{align}
\left\{u \in \R^d: \Vert u \Vert_{*, \Theta} \leq 1 \right\} 
&= \left\{u \in \R^d: \max_{\ind=1}^m \sum_{i=1}^d \gamma^\ind_i u_i^2 \leq 1 \right\} \notag\\
&= \left\{u \in \R^d:  \sum_{i=1}^d \gamma^\ind_i u_i^2 \leq 1, \forall \ind \in \NN{m} \right\} \notag\\
&= \bigcap_{\ind \in \NN{m}}  \left\{u \in \R^d:  \sum_{i=1}^d \gamma^\ind_i u_i^2 \leq 1 \right\} \label{eqn:dual-unit-ball}.
\end{align}
Notice that for each $\ind \in \NN{m}$, the set $\left\{u \in \R^d:  \sum_{i=1}^d \gamma^\ind_i u_i^2 \leq 1 \right\}$ defines a (possibly degenerate) ellipsoid in $X$, where the $i$-th semi-principal axis has length $1/\sqrt{\gamma^\ind_i}$ (which is infinite if $\gamma^\ind_i=0$) and the unit ball of the dual $\Theta$-norm is given by the intersection of $m$ such ellipsoids. 

The following result, which is discussed in \cite[Section 2]{Argyriou2012} is key for the proof of Proposition~\ref{prop:spectral-unit-ball}. 
\begin{corollary}\label{prop:vector-ksup-Ck-def}
The unit ball of the vector $k$-support norm is equal to the convex hull of the set $\{w \in \mathbb{R}^d : {\rm card}(w) \leq k, \Vert w \Vert_2\leq 1\}$.
\end{corollary}
\begin{proof}
The result follows directly by Corollary \ref{cor:GLO} for $\G = \G_k$ observing that in this case
$\bigcup_{g \in \G_k} \left\{w \in \R^d: \supp(w) \subseteq g,\|w\|_2 \leq 1\right\}= \{w \in \R^d: \card(w) \leq k,\|w\|_2\leq 1\}$.
\end{proof}

The next result is due to \citet{VonNeumann1937}; see also \citet{Lewis1995}.
\begin{theorem}[Von Neumann's trace inequality]\label{thm:vonneumann}
For any $d\times m$ matrices $X$ and $Y$, 
\begin{align*}
\tr (X Y\trans) \leq \langle \sigma(X), \sigma(Y) \rangle
\end{align*}
and equality holds if and only if $X$ and $Y$ admit a simultaneous singular value decomposition, that is, $
X = U {\rm diag}(\sigma(X))V\trans$, $Y= U {\rm diag}(\sigma(Y))V\trans$, where $U\in \mathbb{R}^{d \times d}$ and $V\in\mathbb{R}^{m \times m}$ are orthogonal matrices. 
\end{theorem}

{The following inequality is given in \citet[Sec. 9 H.1.h]{Marshall1979}.  

\begin{lemma}
\label{lem:olkin}
If $A,B \in {\bf S}^d_+$, then it holds
\begin{align*}
\tr(AB) = \sum_{i=1}^d \lambda_i(AB) \geq \sum_{i=1}^d \lambda_i(A) \lambda_{d-i+1}(B).
\end{align*}
\end{lemma}


}

\end{document}